\documentclass[12pt,a4paper]{article}

\usepackage{amssymb}
\usepackage{graphicx}
\usepackage{url}
\urldef{\mailsa}\path|{yaflalo,ron}@cs.technion.ac.il|
\usepackage{float}
\usepackage{subfigure}
\usepackage{graphicx}
\usepackage{listings}
\DeclareGraphicsExtensions{.jpg,.eps,.png}
\usepackage{amsmath}
\usepackage{amsfonts}
\usepackage{amssymb}
\usepackage{amsthm}
\usepackage{enumerate}
\usepackage{textcomp}
\usepackage{color}
\usepackage{algorithmic}
\usepackage{algorithm}
\usepackage{bm}
\usepackage{dsfont}
\newcommand{\ones}{\mathds{1}}

\newtheorem{theorem}{Theorem}[section]

\newtheorem{lemma}{Lemma}[section]

\newcommand{\argmin}{\operatornamewithlimits{argmin~}}

\newcommand{\trace}{\operatorname{trace}}

\newcommand{\spn}{\operatorname{span}}
\newcommand{\im}{\operatorname{im}}
\newenvironment{disarray}%
 {\everymath{\displaystyle\everymath{}}\array}%
 {\endarray}
\everymath{\displaystyle\everymath{}}

\newcommand{\D}{\mathbf{D}}

\newcommand{\X}{\mathbf{X}}

\newcommand{\J}{\mathbf{J}}

\newcommand{\PHI}{\mathbf{\Phi}}

\newcommand{\RR}{\mathbb{R}}

\newcommand{\Alpha}{\bm{\alpha}}
\newcommand{\Beta}{\bm{\beta}}

\title{On the optimality of shape and data representation in the spectral domain}
\author{Yonathan Aflalo
\thanks{Computer Science Department, Technion, Israel Institute of Technology, Haifa 32000, Israel}
\and Haim Brezis
\thanks{
Rutgers University, Department of Mathematics, Hill Center, Busch Campus, 110 Frelinghuysen Road, 
   Piscataway,  NJ 08854, USA,
 and
 Department of Mathematics, Technion, Israel Institute of Technology, 32000 Haifa, Israel
}
\and Ron Kimmel
\footnotemark[1]
}

\begin{document}
\maketitle
\abstract{
A proof of the optimality of the eigenfunctions of the Laplace-Beltrami operator (LBO)
 in representing smooth functions on surfaces is provided and adapted to the 
 field of applied shape and data analysis.
It is based on the Courant-Fischer  min-max principle adapted to our case.
The theorem we present supports the new trend in geometry processing of 
 treating geometric structures by using their projection onto the leading eigenfunctions 
  of the decomposition of the LBO.
Utilization of this result can be used for constructing numerically efficient algorithms
 to process shapes in their spectrum.
We review a couple of applications as possible practical usage cases
 of the proposed optimality criteria.
We refer to a scale invariant metric, which is also invariant to bending of the manifold.
This novel pseudo-metric allows constructing an LBO by which a scale 
 invariant eigenspace on the surface is defined.
We demonstrate the efficiency of an intermediate metric, defined as an interpolation
 between the scale invariant and the regular one, in representing geometric
 structures while capturing both coarse and fine details.
Next, we review a numerical acceleration technique for classical scaling, a member 
 of a family of flattening methods known as multidimensional scaling (MDS).
There, the optimality is exploited
 to efficiently approximate all geodesic distances between pairs of points on a given
 surface,  and thereby match and compare between almost isometric surfaces.
Finally, we revisit the classical principal component analysis (PCA) definition
 by coupling its variational form with a Dirichlet energy on the data manifold.
By pairing the PCA with the LBO we can handle cases that go beyond the 
 scope defined by the observation set that is handled by regular PCA.
 }
 
\section{Introduction}
The field of shape analysis has been  evolving rapidly
 during the last decades.
The constant increase in computing power allowed image and shape understanding
  algorithms to efficiently handle difficult problems that could not have been 
  practically addressed in the past. 
A large set of theoretical tools from metric geometry, differential geometry, and 
 spectral analysis has been imported and translated into action within the shape
 and image understanding arena.
Among the myriad of operators recently explored, the {\it Laplace-Beltrami operator} 
 (LBO) is ubiquitous. 
The LBO is an extension of the Laplacian to non-flat multi-dimensional manifolds. 
Its properties have been well studied in differential geometry and it was used 
 extensively in computer graphics. 
It is used to define the {\it heat equation}, that models
  the conduction of heat in solids, and is fundamental in describing basic physical phenomena.
In its more general setting, the Laplace-Beltrami operator admits an eigen-decomposition 
 that yields a {\it spectral domain} that can be viewed as a generalization of the 
 Fourier analysis to any Riemannian manifold. 
The LBO invariance to isometric transformations allowed the theories developed 
 by physicists and mathematician to be useful for modern shape analysis.
Here, we justify the selection of the leading eigenfunctions in the spectral 
 domain as an optimal sub-space for representing smooth functions 
 on a given manifold.
It is used for solving and accelerating existing solvers of various problems in 
 data representation, information processing, and shape analysis. 
As one example, in Section \ref{sec:scale}  we pose the dilemma of metric 
 selection for shape representation while interpolating between a scale invariant metric 
 and the regular one. 
Next, in Section \ref{sec:SMDS} it is shown how the recently introduced 
 spectral classical scaling can benefit from the efficacy property of the suggested subspace.
Finally, in Section \ref{sec:RPCA} we revisit the definition of the celebrated 
 {\it principal component analysis} (PCA)  by regularizing its variational form
 with an additional Dirichlet energy. 
The idea is to balance between two optimal sub-spaces,
 one for the data points themselves - captured by the PCA,
 and one optimally encapsulating the relation between the data points 
 as defined by decomposition of the LBO. 

\section{Notations and motivation}
Consider a parametrized surface $S:\Omega\subset\mathbb{R}^2 \rightarrow \mathbb{R}^3$ (with or without boundary) 
 and a metric $(g_{ij})$
 that defines the affine differential relation of navigating with coordinates $\{\sigma_1,\sigma_2\}$ 
 in $\Omega$ to a distance measured on $S$.
That is, an arc length on $S$ expressed by $\sigma_1$ and $\sigma_2$ would read 
 $ds^2 = g_{11}d\sigma_1^2+2g_{12}d\sigma_1d\sigma_2+g_{22}d\sigma_2^2$.
The Laplace-Beltrami operator acting on the scalar function $f:S\rightarrow \mathbb{R}$ 
 is defined as
$$
 -\Delta_g f=\frac{1}{\sqrt{g}}\sum_{ij}\partial_i\left(\sqrt{g}g^{ij}\partial_j f\right),
$$
 where $g$ is the determinant of the metric matrix, and $(g^{ij})=(g_{i,j})^{-1}$ is the inverse metric, 
 while $\partial_i$ is a derivative with respect to the $i^\text{th}$ coordinate $\sigma_i$.
The LBO operator is symmetric and admits a spectral decomposition $(\lambda_i,\phi_i)$, with $\lambda_1\le\lambda_2\le...$, such that
 \begin{eqnarray*}
 \Delta_g \phi_i&=& \lambda_i\phi_i\cr
 \langle \phi_i,\phi_j\rangle &=&  \delta_{ij},
 \end{eqnarray*}
 where $\langle u,v\rangle=\int_S uv\sqrt{g}dx$, and $\|u\|_2^2 = \int_S |u|^2\sqrt{g}dx$.
In case $S$ has a boundary, we add Neumann boundary condition  
\[
 \frac{\partial \phi_i}{\partial \nu} = 0, \,\,\,\, \mbox{on}\,\,\, \partial S.
\]
Defined by the metric rather than the explicit embedding, makes the LBO and its spectral decomposition 
 invariant to isometries and thus, a popular operator for shapes processing and analysis.
For example, the eigenfunctions and eigenvalues can be used to efficiently approximate
 diffusion distances and commute time distances
  \cite{commute_time,Berard1994,Coifman_Lafon:2006,coifman_PNAS2005,bronstein2010gromov}, 
 that were defined as computational alternatives to geodesic distances, and were shown
  to be robust to topology changes and global scale transformations. 
At another end,  L\'evy \cite{levy2006laplace} proposed to manipulate the geometry of shapes
 by operating in their spectral domain, while Gotsman and Karni \cite{karni00spectral} chose the eigenfunctions 
 as a natural basis for approximating the coordinates of a given shape.
Feature point detectors and descriptors of surfaces were also extracted
  from the same spectral domain.
Such measures include the heat kernel signature (HKS) \cite{HKS,gebal2009shape}, the
  global point signature (GPS) \cite{Rustamov13}, 
 the wave kernel signature (WKS) \cite{aubry2011wave}, 
  and the scale-space representation \cite{zaharescu2009surface}.  

Given two surfaces $S$ and $Q$, and a bijective mapping between them, $\rho:S\rightarrow Q$, 
 Ovsjanikov et al. \cite{fmap} emphasized the fact that 
 %
 the relation between the spectral decomposition of a scalar function $f:S\rightarrow \RR$ and 
 and its representative on $Q$, that is $f\circ \rho^{-1}:Q\rightarrow \RR$,
 is linear.
%
In other words, the geometry of the mapping is captured by $\rho$, allowing the coefficients of the 
 decompositions to be related in a simple linear manner. 
The  basis extracted from the LBO was chosen in this context because of its intuitive efficiency in 
 representing functions on manifolds, thus far justified heuristically. 
The linear relation between the spectral decomposition coefficients of the same 
 function on two surfaces, when the mapping between manifolds is provided,
 was exploited by Pokrass et al. \cite{pokrass} to find the correspondence between
 two almost isometric shapes.  
They assumed that the matrix that links between the LBO eigenfunctions
  of two almost isometric shapes should have dominant coefficients along its diagonal, a 
  property that was first exploited in \cite{joint_diagonalization}.

One could use the relation between the eigen-structures of two surfaces 
 to approximate non-scalar and non-local structures on the manifolds \cite{SGMDS}.
Examples for such functions are geodesic distances 
 \cite{Kim_Seth, Tsitsiklis,Fast_mrch,MMP, Surazhsky:2005}, 
 that serve as an input for
 the {\em Multi-Dimensional Scaling} \cite{MDS_review_book:1997}, 
   the {\em Generalized Multi-Dimensional Scaling} \cite{GMDS},
    and the Gromov-Hausdorff distance  \cite{memoli_sapiro, bronstein2006efficient}.
Using the optimality of representing surfaces and functions on surfaces
  with truncated basis, geodesic distances can now be efficiently 
  computed and matched in the spectral domain.
 
Among the myriads reasons that motivate the choice of 
 the spectral domain for shape analysis,  we emphasize the following, 
\begin{itemize}
\item The spectral domain is isometric invariant.
\item Countless  signal processing tools  that exploit the Fourier basis are available.
        Some of these tools can be generalized to shapes for processing, analysis,
         and synthesis.
\item Most interesting functions defined on surfaces are smooth and can thus 
 be approximated by their projection onto a small number of eigenfunctions.
\item For two given perfectly isometric shapes, the problem of finding correspondences 
 between the shapes appears to have a simple formulation in the spectral domain.
\end{itemize}
Still, a rigorous justification for the selection of the basis defined by the LBO
 was missing in the shape analysis arena.
Along the same line, combing the eigenstructure of the LBO with 
 classical data representation 
 and analysis procedures that operate in other domains like the PCA \cite{PCA}, 
 MDS \cite{MDS_review_book:1997},  and GMDS \cite{GMDS} was yet to come.
Here, we review recent improvements of existing tools that make use 
 of the decomposition of Laplace-Beltrami operator. 
We provide a theoretical justification for using the LBO eigen-decomposition 
 in many shape analysis methods. 
With this property in mind, we demonstrate that it is possible to migrate
 algorithms to the spectral domain while establishing a substantial 
 reduction in complexity. 

\section{Optimality of the LBO eigenspace}
\label{sec:theorm}

In this section we provide a theoretical justification to the choice of the LBO eigenfunctions, 
 by proving that the resulting spectral decomposition is optimal in approximating 
 functions with $L^2$ bounded gradient magnitudes.
Let $S$ be a given Riemannian manifold with a metric $(g_{ij})$, an induced LBO, $\Delta_g$, 
 with associated spectral basis $\phi_i$, where $\Delta_g\phi_i=\lambda_i\phi_i$.
It is shown, for example in \cite{journals/PNAS/AflaloK13},
 that for any $f:S\rightarrow \mathbb{R}$, the representation error
\begin{eqnarray}
\label{eq:rep_err}
 \|r_n\|_2^2 \equiv \left\|f-\sum_{i=1}^n\langle f, \phi_i \rangle \phi_i\right\|_2^2\leq
  \frac{\|\nabla_g f\|_s^2}{\lambda_{n+1}}.
\end{eqnarray}
%
Our next result asserts that the eigenfunctions of the LBO are optimal with respect 
 to estimate error (\ref{eq:rep_err}).
%
\begin{theorem}
 Let  $0\leq\alpha<1$.
There is no
 integer $n$ and no sequence $\{\psi_i\}_{i=1}^n$ of linearly 
 independent functions in $L^2$ such that 
\begin{eqnarray}
\label{eq:bound1}
 \left \|f -\sum_{i=1}^n \langle f,\psi_i\rangle \psi_i\right \|_2^2 \le 
 \frac{\alpha \|\nabla_g f\|_2^2}{\lambda_{n+1}},\,\,\,\,\,\,\,\,\,\, \forall f.
\end {eqnarray}
\begin{proof}
Recall the Courant-Fischer min-max principle,
 see \cite{brezis} Problems 37 and 49, and \cite{Weinberger74}.
We have for every $n\ge0$,
\begin{eqnarray}
\label{eq:lambda}
\lambda_{n+1} &=& \max_{\substack{ \Lambda\\ \mbox{\scriptsize codim} \,\Lambda=n}}
\min_{\substack{f\in\Lambda\\ f\ne 0}}\left \{
    \frac{\|\nabla_g f\|_2^2}{\|f\|_2^2} \right \}.
\end {eqnarray}
That is, the $\min$ is taken over a linear subspace 
$\Lambda \subset H^1(S)$
 (where $H^1(S)$ is the Sobolev space $\{ f\in L^2, \, \nabla_gf\in L^2\}$) of co-dimension $n$
 and the $\max$ is taken over all such subspaces.

Set $\Lambda_0 = \{f\in H^1(S); \,\, \langle f,\psi_i\rangle =0, \,\, i=1,2,...,n\}$,
 so that $\Lambda_0$ is a subspace of codimension $n$.
By (\ref{eq:bound1}) we have $\forall f\ne0,\,\, f\in\Lambda_0$,
\[
\frac{ \|\nabla_g f\|_{2}^2}{\|f\|_{2}^2}\ge \frac{\lambda_{n+1}}{\alpha},
\]
 and thus
\begin{eqnarray}
\label{eq:X0}
X_0 = \min_{\substack{f\in\Lambda_0\\ f\ne0}}
 \left \{ 
 \frac{ \|\nabla_g f\|_{2}^2}{\|f\|_{2}^2}
 \right \}\ge \frac{\lambda_{n+1}}{\alpha}.
\end {eqnarray}
On the other hand, by (\ref{eq:lambda})
 \begin{eqnarray}
 \label{eq:limit}
	\lambda_{n+1}\ge X_0.
\end {eqnarray}
Combining (\ref{eq:X0}) and (\ref{eq:limit}) yields $\alpha \ge 1$.
\end{proof}
\end{theorem}

For the convenience of the reader we present in the appendix 
 a direct proof of a special case of the above result which does not make use of 
 the Courant-Fischer min-max principle.
The above theorem proves the optimality of the eigenfunctions of 
 the LBO in representing $H^1$ functions on manifolds.
In the following sections we apply the optimality property 
 for solving various shape analysis problems. 

\section{Scale invariant geometry} 
\label{sec:scale}
Almost isometric transformation are probably the most common
 ones for surfaces and volumes in nature.
Still, in some scenarios, relations between surfaces  should  
 be described by slightly more involved deformation-models.
Though a small child and an adult are obviously not isometric, 
 and the same goes for a whale and a dolphin,
 the main characteristics are morphometrically similar for mammals in large. 
In order to extend the scope of matching and comparing shapes, 
 a semi-local scale invariant geometry was introduced in \cite{journals/siamis/AflaloKR13}.
There, it was used to define a new LBO by which one can construct an eigenspace
 which is invariant to semi-local and obviously global scale transformations.
  
Let $(g_{ij})$, be the regular metric defined on the manifold. 
In \cite{journals/siamis/AflaloKR13}  the  {\em scale invariant pseudometric} $(\tilde g_{ij})$ 
 is defined as
$$
 \tilde g_{ij}=|K| g_{ij},
$$
  where $K$ is the Gaussian curvature at each point on the manifold. 
One could show that this metric is scale invariant and the same goes for the LBO that it induces,
 namely $\Delta_{\tilde{g}}f=-\frac{1}{\sqrt{\tilde g}}\sum_{ij}\partial_i\left(\sqrt{\tilde g}\tilde{g}^{ij}\partial_j f\right)$. 
A discretization of this operator
 and  experimental results that outperformed state of the art algorithms for shape matching,
 when scaling is involved,
 were presented in  \cite{journals/siamis/AflaloKR13} .
Specifically, the scale invariant geometry allows to find correspondence between two shape related by 
 semi-local scale transformation. 
 
Next, one could think of searching for an optimal representation space for shapes
  by interpolating  between the scale-invariant metric and the regular one.
We define the interpolated pseudometric to be
$$
\hat{g}_{ij} = |K|^\alpha g_{ij},
$$
 where $(\hat{g}_{ij})$ represents the new  pseudometric, $K$ is the Gaussian curvature, 
 and $\alpha\in[0,1]$ is the metric interpolation scalar 
 that we use to control the representation error.
In our setting,  $\hat{g}$ depends on $\alpha$ and represents the regular metric 
 when $\alpha=0$, or the scale invariant one for $\alpha=1$.
 
Figure \ref{fig:horse} depicts the effect of representing a shape's coordinates projected to
 the first 300 eigenfunction of the LBO with a  regular metric  (left), 
 the  scale invariant one (right), and the
   interpolated pseudometric with $\alpha=0.4$ (middle).
The idea is to use only part of eigenfunctions to approximate smooth functions
 on the manifold, treating the coordinates as such.
While the regular natural basis captures the global structure of the surface, as expected, the
 scale invariant one concentrates on the fine features with effective curvature. 
The interpolated one
 is a good compromise between the global structure and the fine details. 
\begin{figure}[htbp]
\centering
\hspace{2cm} \includegraphics[width=0.37\linewidth]{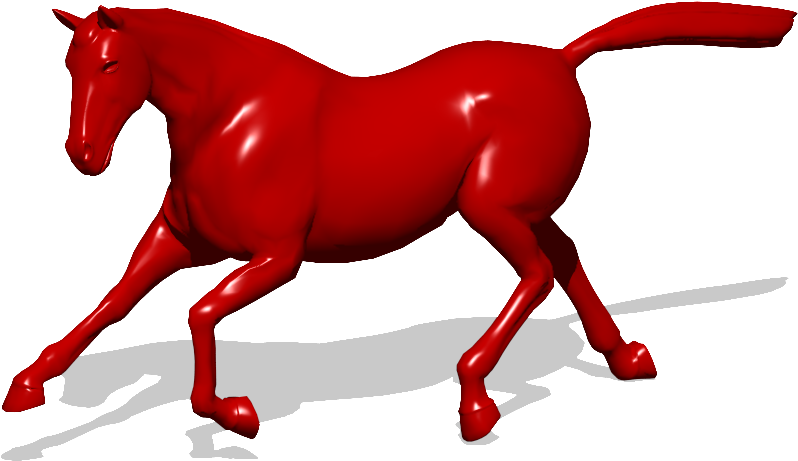}\\
\centering
\mbox{
 \includegraphics[width=0.37\linewidth]{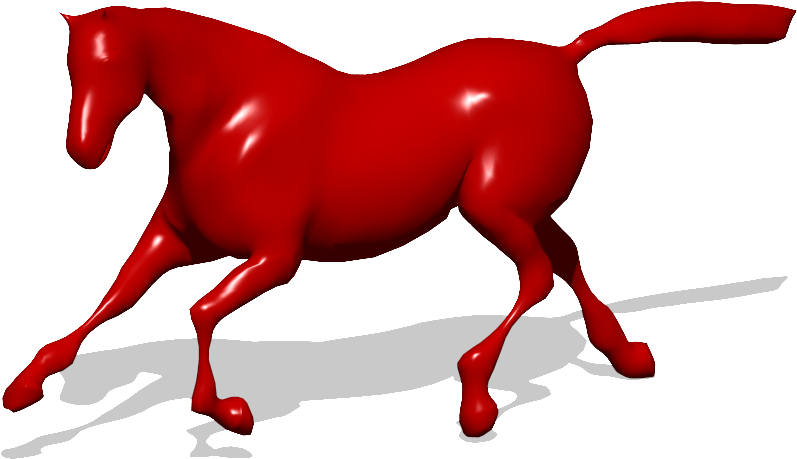}
\includegraphics[width=0.37\linewidth]{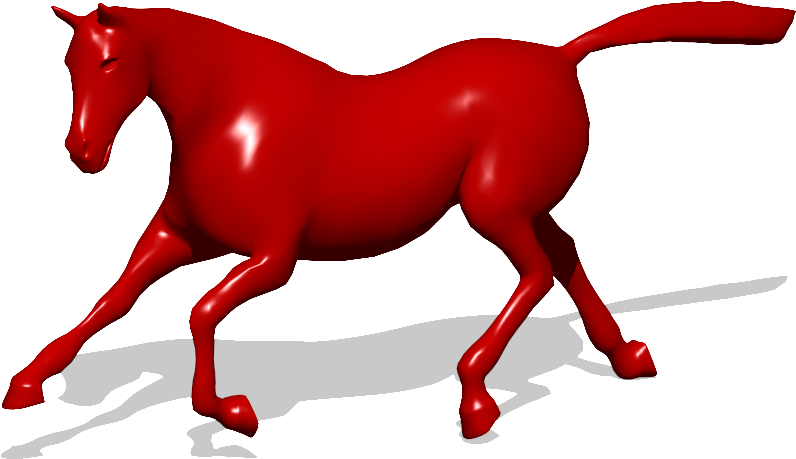}
\includegraphics[width=0.37\linewidth]{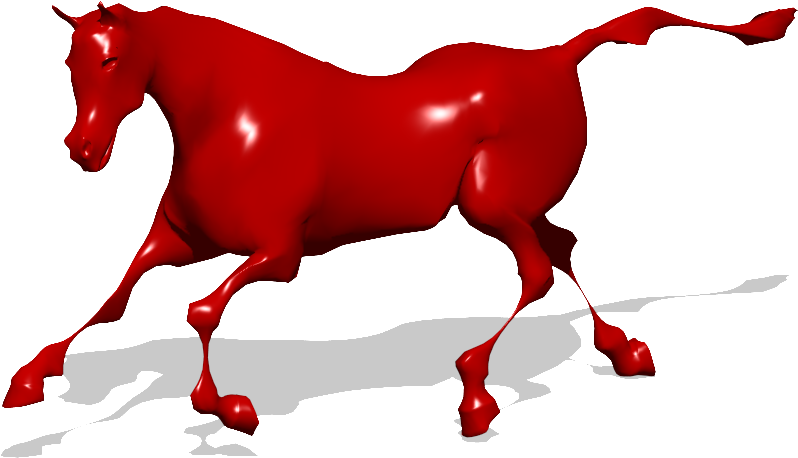}}
\caption{\label{fig:horse}
Top: A given horse model.
Bottom:
 The horse coordinates projected to the first 300 LBO eigenfunctions using a regular metric (left), 
   an intermediate metric (middle), and a scale invariant one (right).
}
\end{figure}

We proved that once a proper metric is defined, the Laplace-Beltrami eigenspace is 
 the best possible space for functional approximation of smooth functions.
Next, we exploit this property to reformulate classical shape analysis algorithms
 such as MDS  in the spectral domain. 
 
 \section{Spectral classical scaling}
 \label{sec:SMDS}
Multidimensional Scaling \cite{MDS_review_book:1997} is a family of data analysis methods 
 that is widely used in machine learning and shape analysis. 
Given an $n\times n$ pairwise distances matrix $\D$, the MDS method finds 
 an embedding of points of $\mathbb{R}^m$, given by an $n\times m$ matrix $\X$,
  such that the pairwise euclidean distances between every two points,
   each defined by a row of $\X$, is as close as possible to their 
   corresponding input pair given by the right entry in  $\D$. 
The classical MDS algorithm minimizes the following functional
 $$
  \X=\argmin_{\X\in\mathbb{R}^{n\times m}}\left\|\X\X^T+\frac{1}{2}\J\D_2\J\right\|_F
 $$
 where $\D_2$ is a matrix such that $(\D_2)_{ij}=\D^2_{ij}$, $\J$ is a centering matrix defined by $\J=\mathbf{I}-\frac{1}{n}\ones\ones^T$, where $\mathbf{I}$ is the identity matrix, $\ones$ is a vector of ones,
  and $\|\cdot\|_F$ is the Frobenius norm.
The solution can be obtained by a singular value decomposition of the matrix $\J\D_2\J$. 
This method was found to be useful when comparing between 
 isometric shapes using their inter-geodesic distances \cite{Elad03,bronstein2006efficient},
 and texture mapping in computer graphics \cite{Zigelman-Kimmel-Kiryati}.
The computation of geodesic distances as well as the SVD of an $n\times n$ matrix 
 can be expensive in terms of memory and computational time.
High resolution shapes with more than $10000$ vertices are difficult do handle with this method. 

In order to reduce the complexity of the problem, it was proposed in \cite{journals/PNAS/AflaloK13}
 to compute geodesic distances between
 a small set of sample points, and then, interpolate the rest of the distances by minimizing
  a bi-harmonic energy in the spectral domain. 
We find a spectral representation of the matrix $\D_2=\PHI\Alpha\PHI^T$,
  where $\PHI$ represents the matrix that discretizes the spectral domain.
 We then embed our problem into the eigenspace of the LBO, defining $\X=\PHI\Beta$, 
  where $\Beta$ is an $m\times k$ matrix, and  $k\ll n$ 
  in order to reduce the overall complexity.
$\X$ is obtained by minimizing 
$$
 \min_{\Beta} \left\|\PHI\Beta\Beta^T\PHI^T+\frac{1}{2}\J\PHI\Alpha\PHI^T\J\right\|_F.
$$
Experimental results of shape canonization comparing  shapes flattened 
 with spectral classical scaling to regular classical scaling results
  were presented in \cite{journals/PNAS/AflaloK13}.
The spectral approach outperformed the classical scaling
 in terms of time and space complexities,
  that lead to overall better accuracy for the spectral version,
 see Figure \ref{fig:MDS}. 
\begin{figure}[htbp]

\mbox{
\includegraphics[width=0.27\columnwidth]{horse2}
\includegraphics[width=0.33\columnwidth]{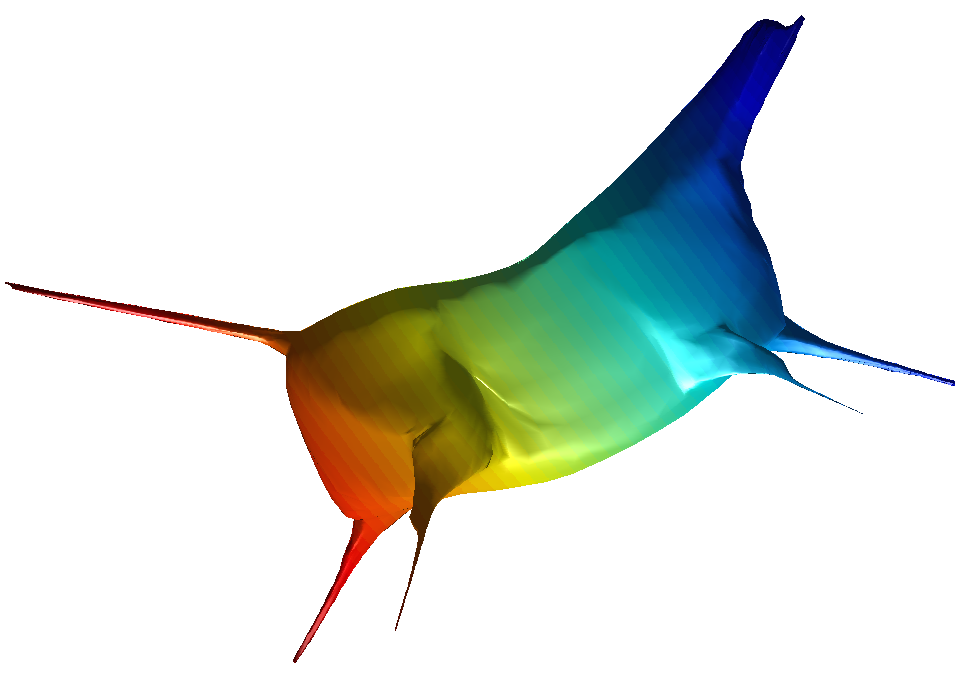}
\includegraphics[width=0.33\columnwidth]{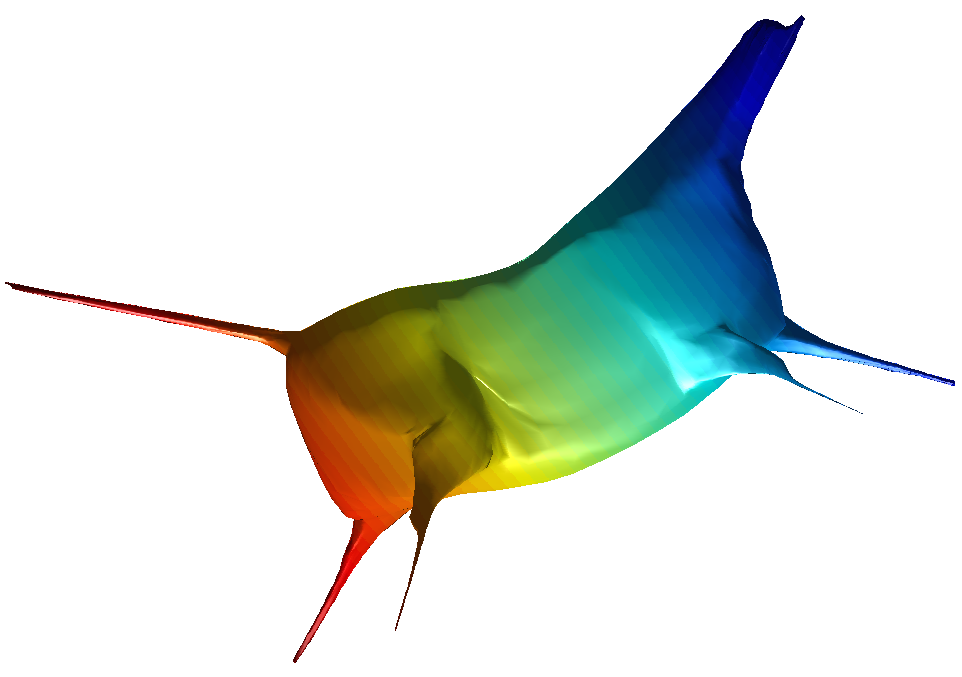}}
\caption{
	MDS flat embedding: Result of regular MDS (middle) and the spectral MDS (right) 
	of the given horse surface (left).
	}
\label{fig:MDS}
\end{figure}
In the next section we introduce a novel design of functional spaces that benefit 
 from both the Laplace-Beltrami operator and  classical {\em principal component analysis}, 
 while extending the scope of each of these measures.
 
\section{Regularized PCA}
\label{sec:RPCA}
The  spectral domain provided by the decomposition of the LBO is 
 efficient in representing smooth functions on the manifold. 
Still, in some scenarios, functions on manifolds could contain discontinuities that 
 do not align with our model assumption. 
Alternatively, some functions could be explicitly provided as known points on the data manifold, 
 in which case, the question of what should be the best representation obtains a new flavor.
The {\it principal component analysis} \cite{PCA} concept allows to extract a low rank 
 orthonormal approximate representation from a set of such data points $x_i$.
Given a set of $k$ vectors $x_i\in\mathbb{R}^n$, the PCA algorithm finds 
  an orthonormal basis of $m\leq k$, defined by its vectors  $P_j$, by minimizing 
$$
\begin{disarray}{l}
\min_\mathbf{P}\sum_{i=1}^k\|\mathbf{P}\mathbf{P}^Tx_i-x_i\|_2^2\\
\text{s.t.}\\
\mathbf{P}^T\mathbf{P}=\mathbf{I}_m.
\end{disarray}
$$
It can be shown that this problem can be written as 
$$
\begin{disarray}{l}
\max_\mathbf{P}\trace(\mathbf{P}\mathbf{P}^T\mathbf{X}\mathbf{X}^T)\\
\text{s.t.}\\
\mathbf{P}^T\mathbf{P}=\mathbf{I}_m,
\end{disarray}
$$
 where $\mathbf{X}$ is a matrix whose $i^{th}$ column is the data point $x_i$. 
At the other end, given a manifold $S$, the spectral basis minimizes 
 the Dirichlet energy of any orthonormal basis defined on $S$, where, 
\begin{equation}
\label{eqn::smooth}
 \begin{disarray}{l}
 \PHI=\argmin_{\{\psi_i\}_{i=1}^n} \sum_{i=1}^n \|\nabla_g \psi_i\|_g^2 \cr
 \text{s.t.}\cr
 \langle \psi_i,\psi_j\rangle_g=\delta_{ij}\,\,\,~~~\forall (i,j),
 \end{disarray}
\end{equation}
 where $\delta_{ij}$ is the Kronecker delta symbol, 
 and $n$ is the number of desired basis functions. 
Using a discretization of the Laplace-Belrami operator,
 it can be shown that the PCA and the computation of a spectral basis could be married.
We can combine both energies, namely, the energy defined by the data projection error 
 and the Dirichlet energy of the representation space. 
The result reads,
\begin{equation}
\label{eqn:pb_smooth_PCA}
\begin{disarray}{l}
   \min_\mathbf{P}\underbrace{\sum_{i=1}^m\|\mathbf{P}\mathbf{P}^T\mathbf{A}x_i-x_i\|_g^2}_{\text{PCA}}
    +\mu\underbrace{\sum_{j=1}^m\|\nabla_g P_j\|_g^2}_{\text{LBO-eigenspace}}\\
   \text{s.t.}\\
\mathbf{P}^T\mathbf{A}\mathbf{P}=\mathbf{I}_m.
\end{disarray}
\end{equation}
Where $\mathbf{A}$ represents the local area normalization factor.
This problem is equivalent to finding a basis that is both 
 smooth and whose projection error on the set of given vectors (data points)
 is minimal, as  shown in \cite{regularized_PCA} and  \cite{GLPCA}. 
When $\mu$ in the above model is set to zero, we have the PCA as a solution. 
At the other end, as $\mu$ goes to infinity, we get back our LBO eigenbasis domain. 
The parameter $\mu$ controls the smoothness of the desired basis. 
The benefits of this hybrid model in representing out of sample information
 are demonstrated in \cite{regularized_PCA}, as can be seen in Figure \ref{fig:RPCA}.


\begin{figure*}[htbp]
\centering
\fbox{
\includegraphics[width=0.17\columnwidth]{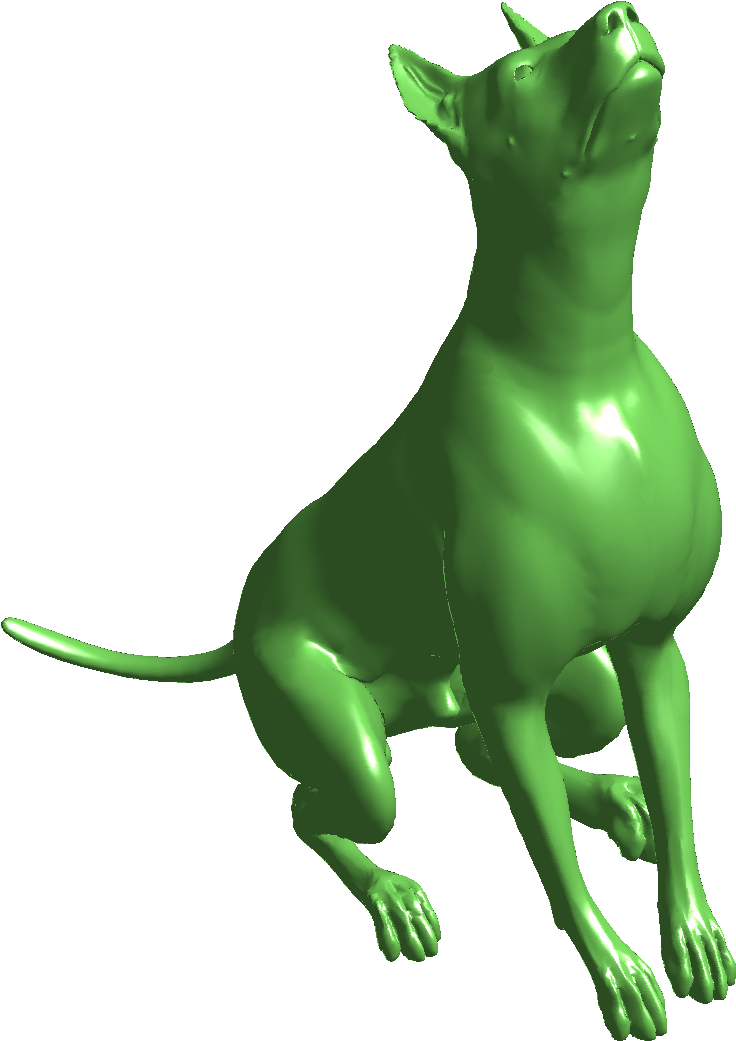}
\includegraphics[width=0.21\columnwidth]{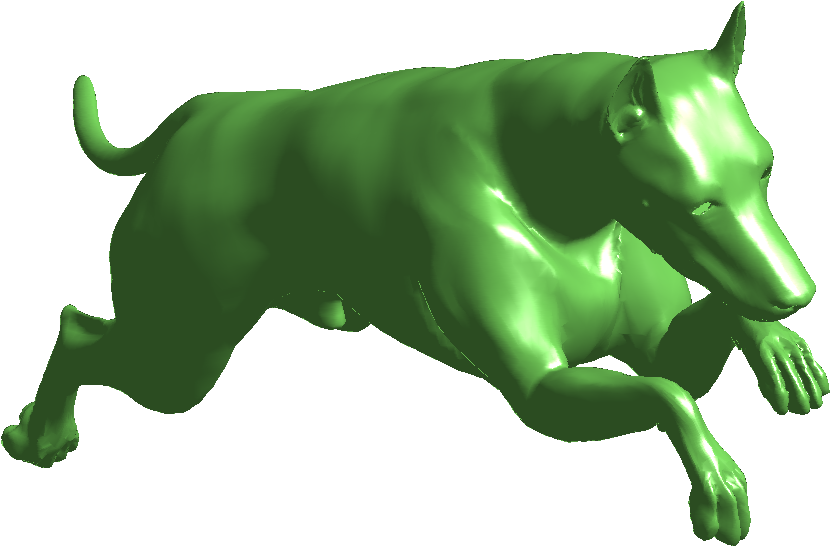}
}
\fbox{
\includegraphics[width=0.22\columnwidth]{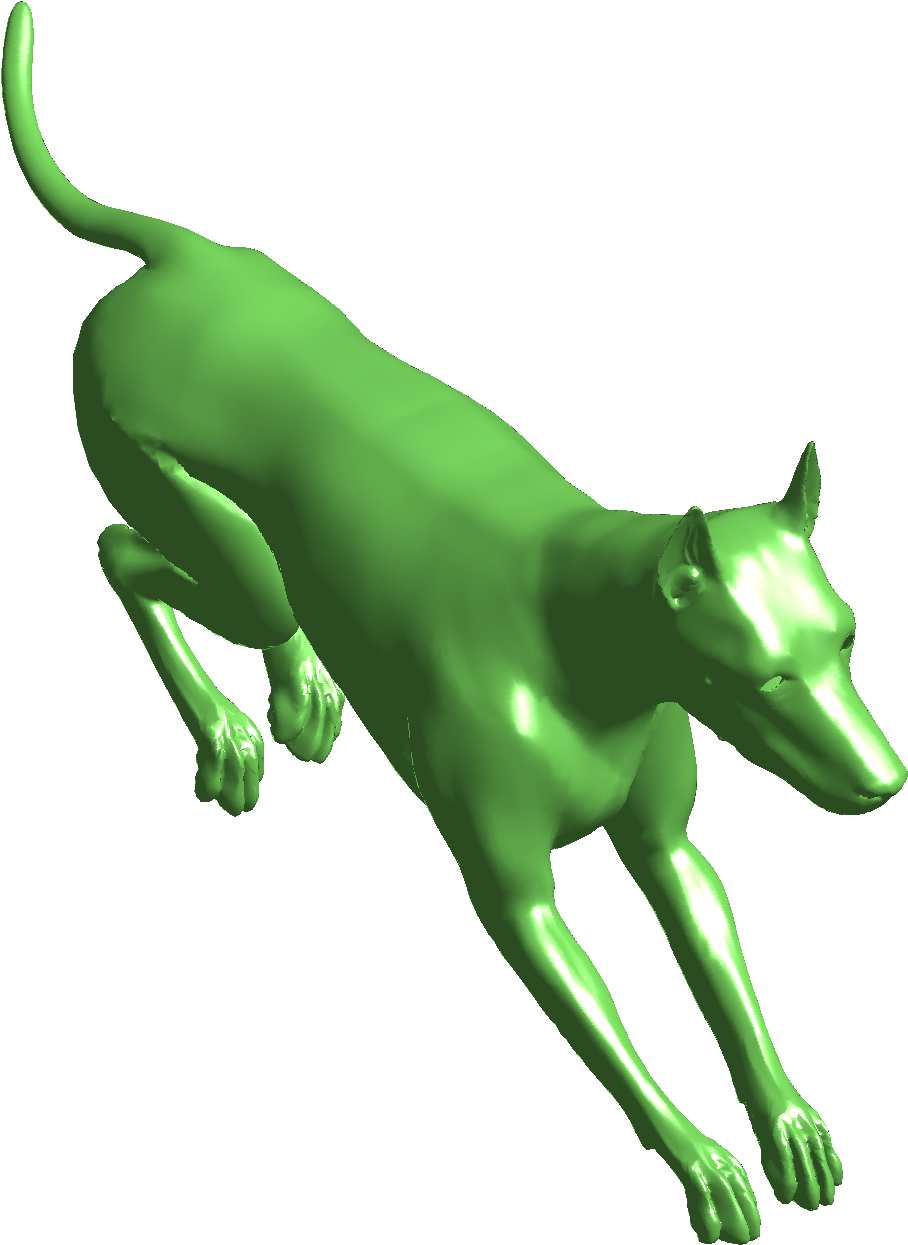}
\includegraphics[width=0.27\columnwidth]{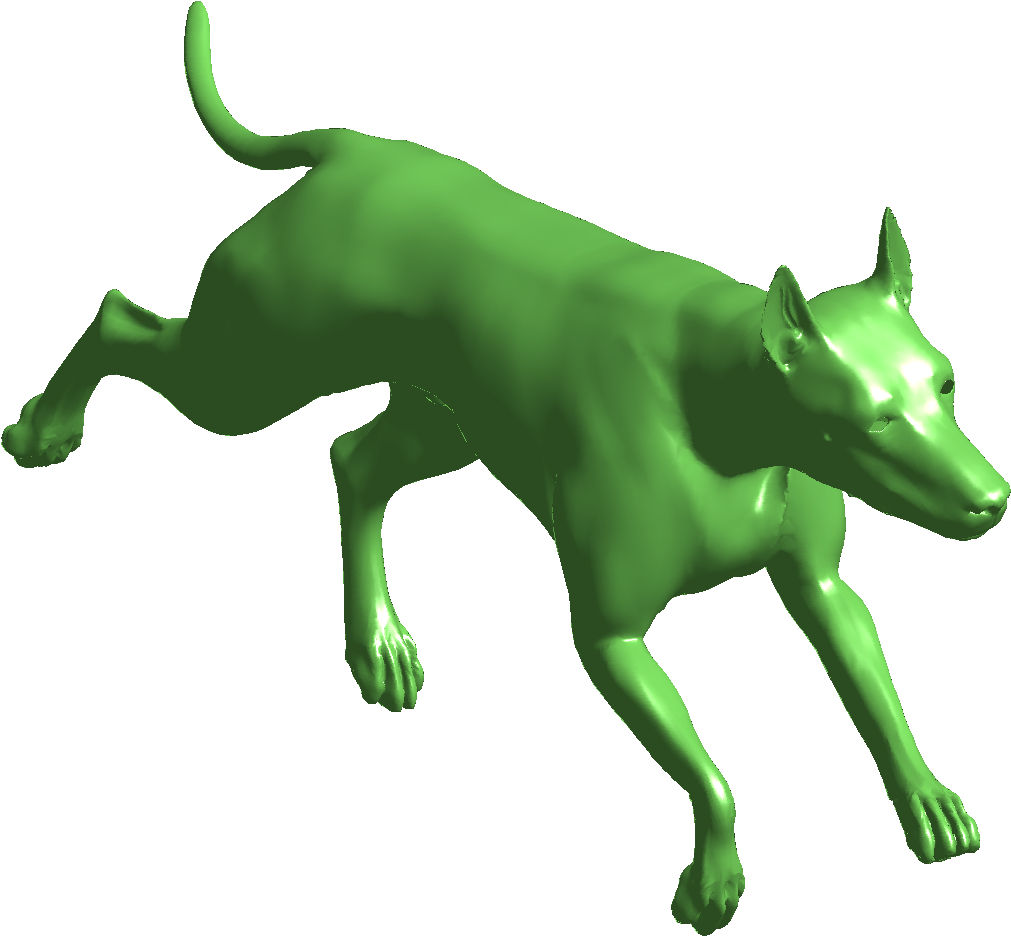}
\includegraphics[width=0.15\columnwidth]{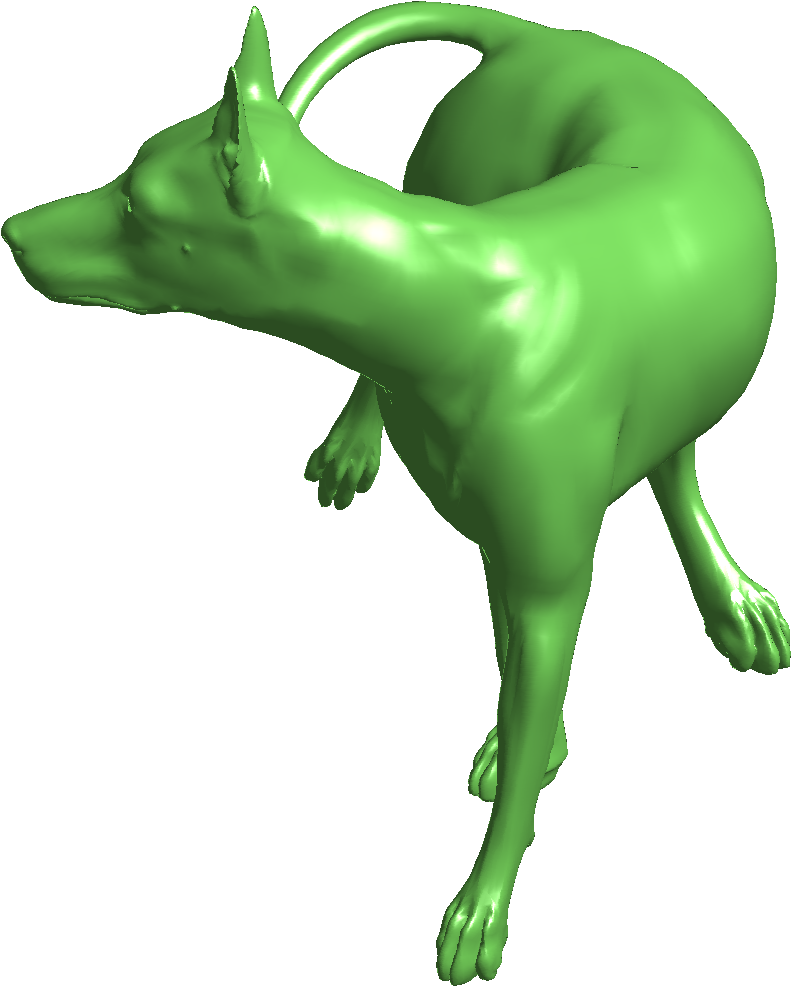}
\includegraphics[width=0.22\columnwidth]{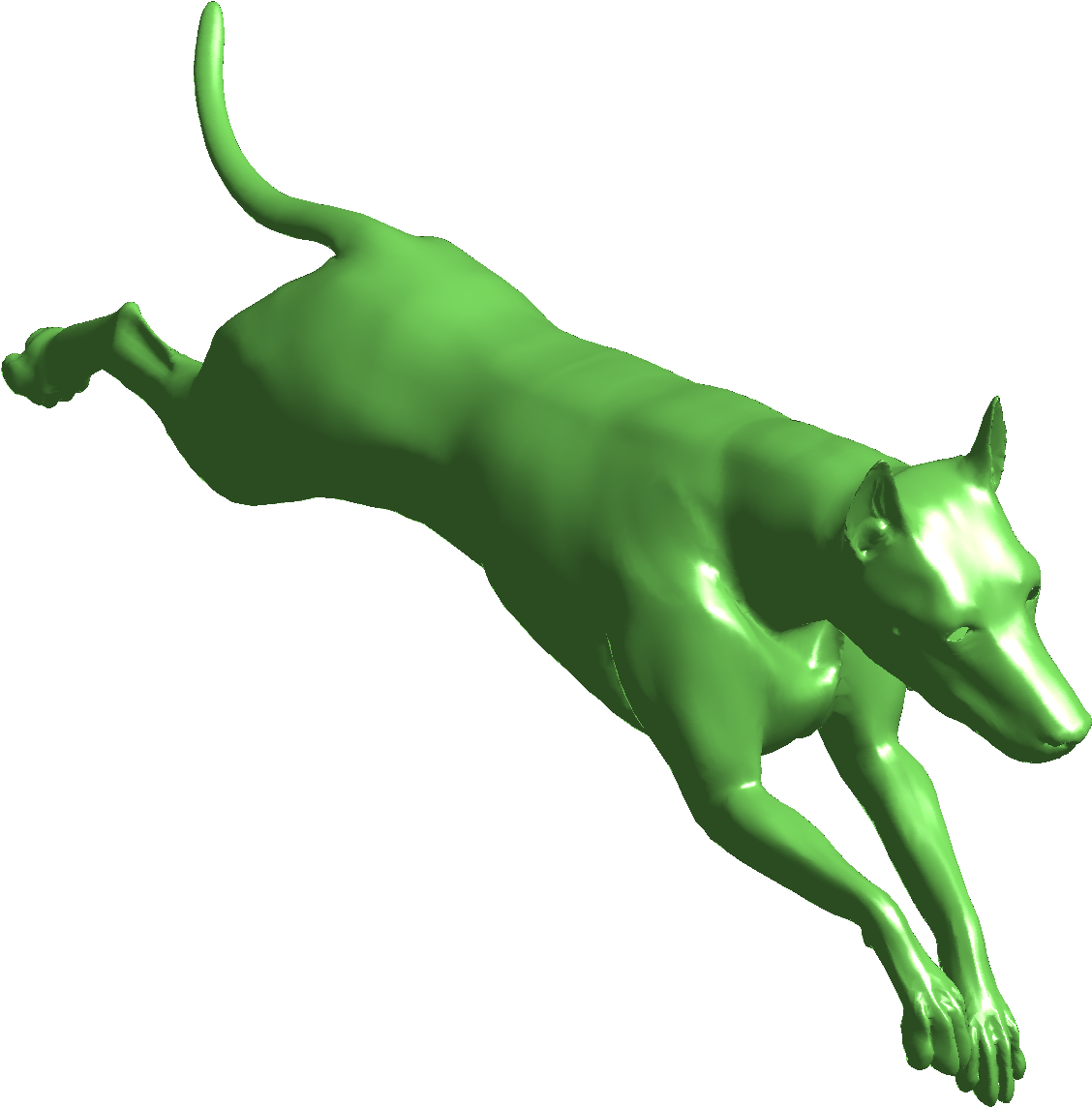}
\includegraphics[width=0.14\columnwidth]{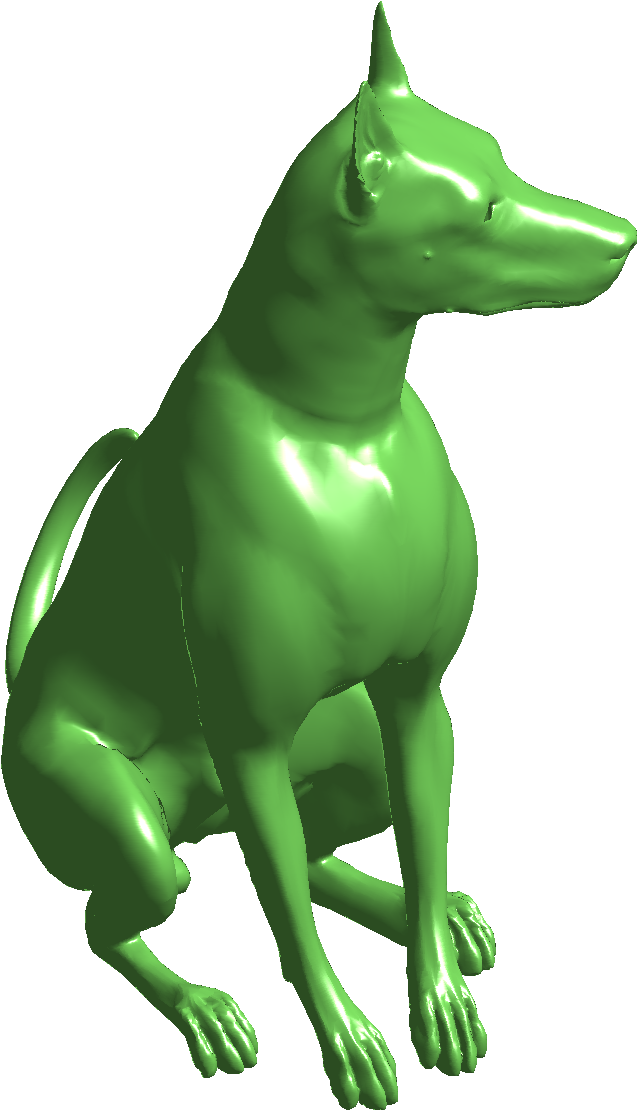}}
\fbox{
\includegraphics[width=0.21\columnwidth]{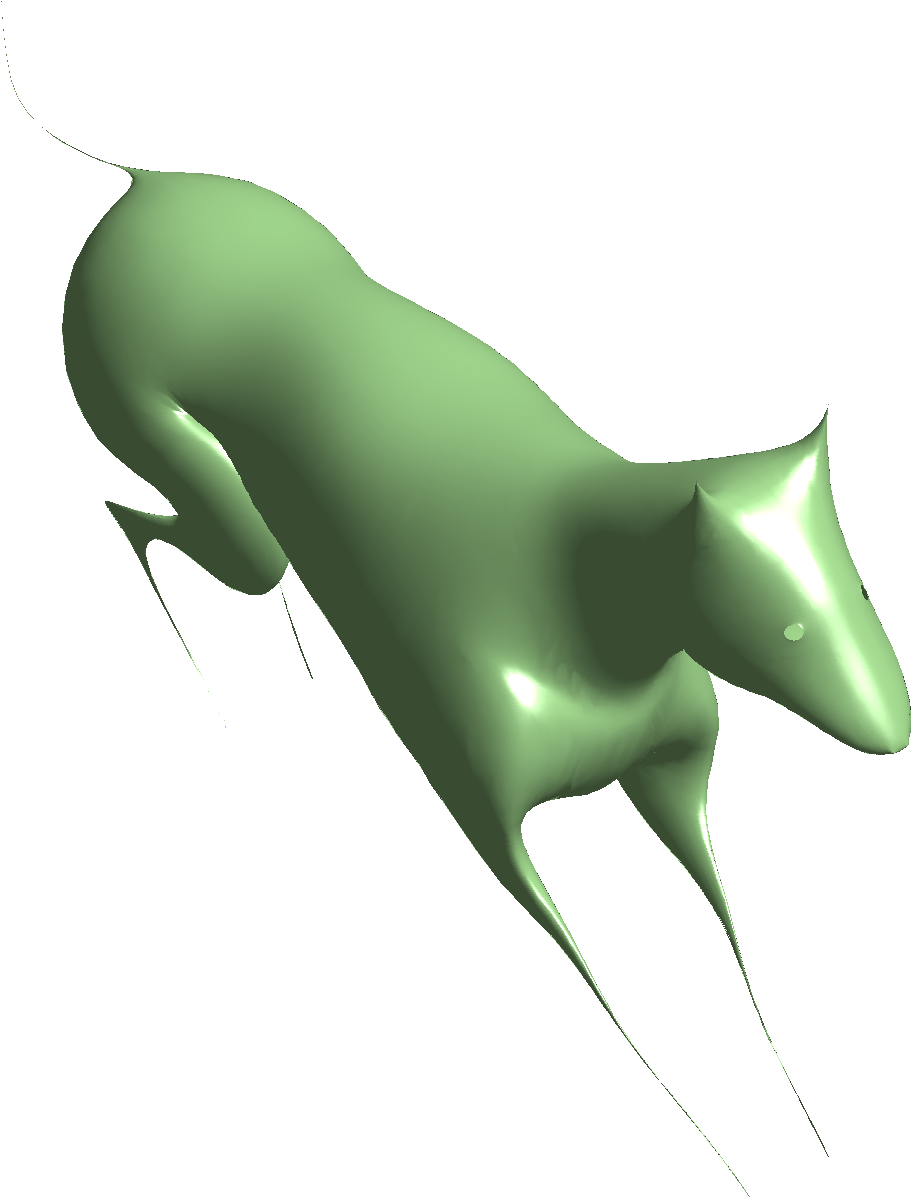}
\includegraphics[width=0.28\columnwidth]{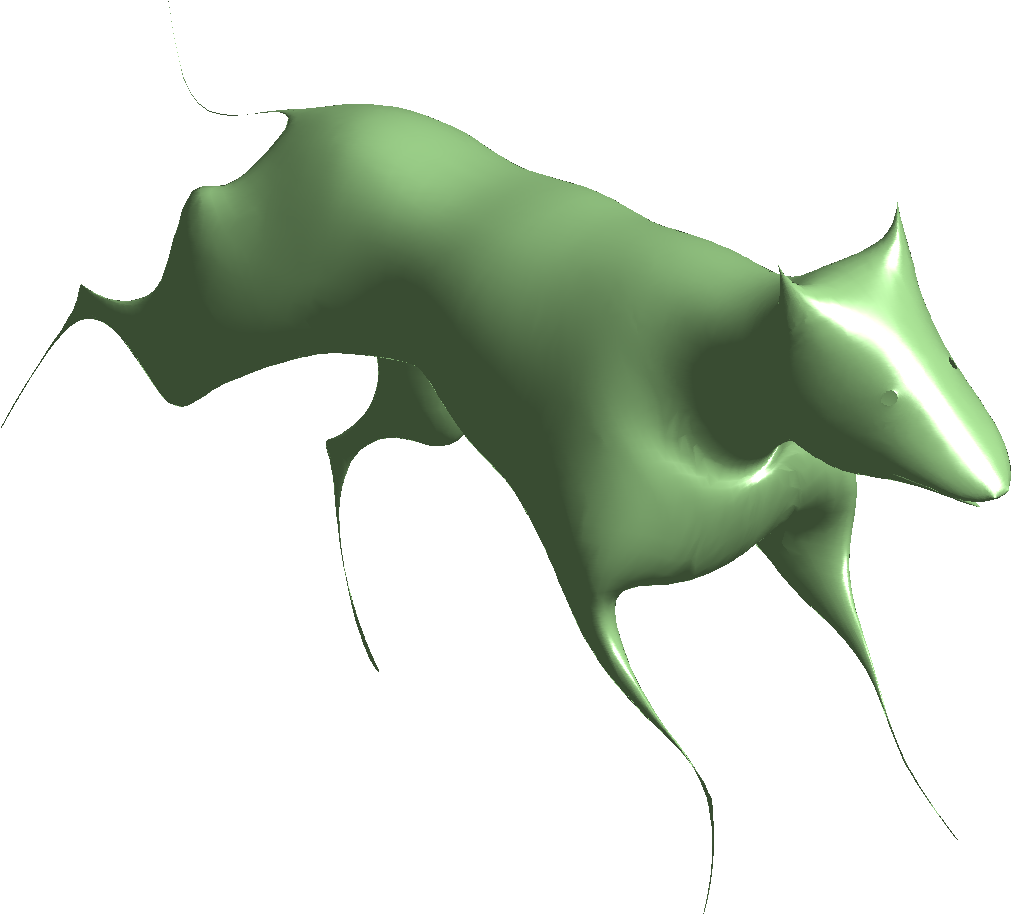}
\includegraphics[width=0.15\columnwidth]{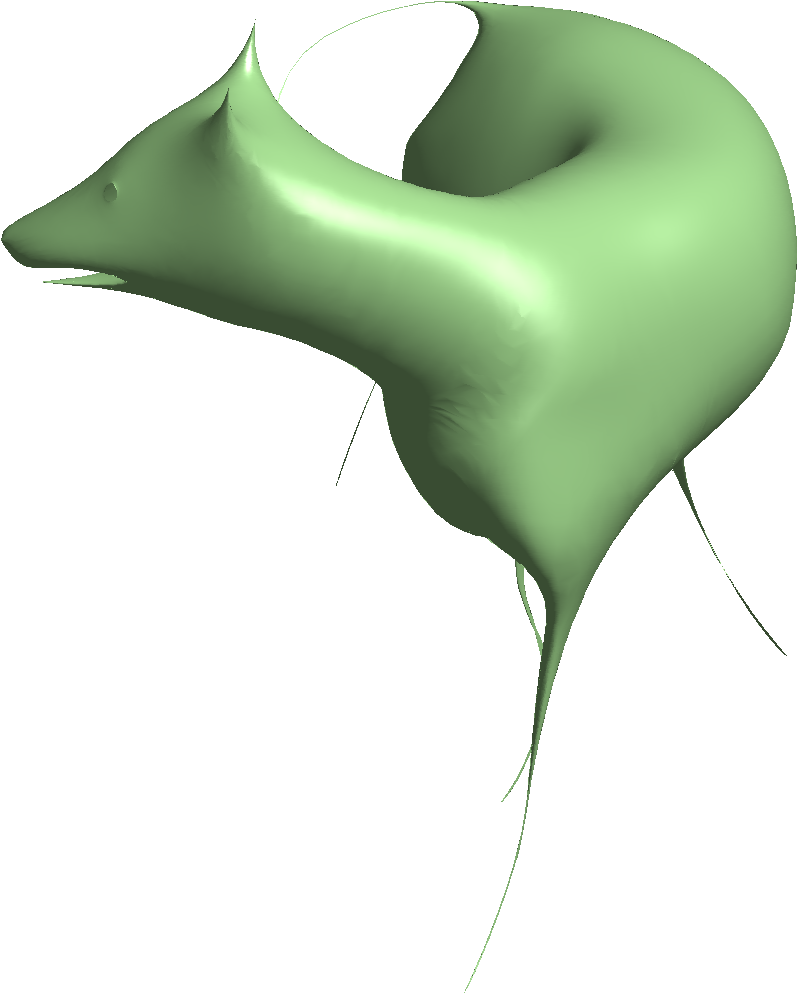}
\includegraphics[width=0.22\columnwidth]{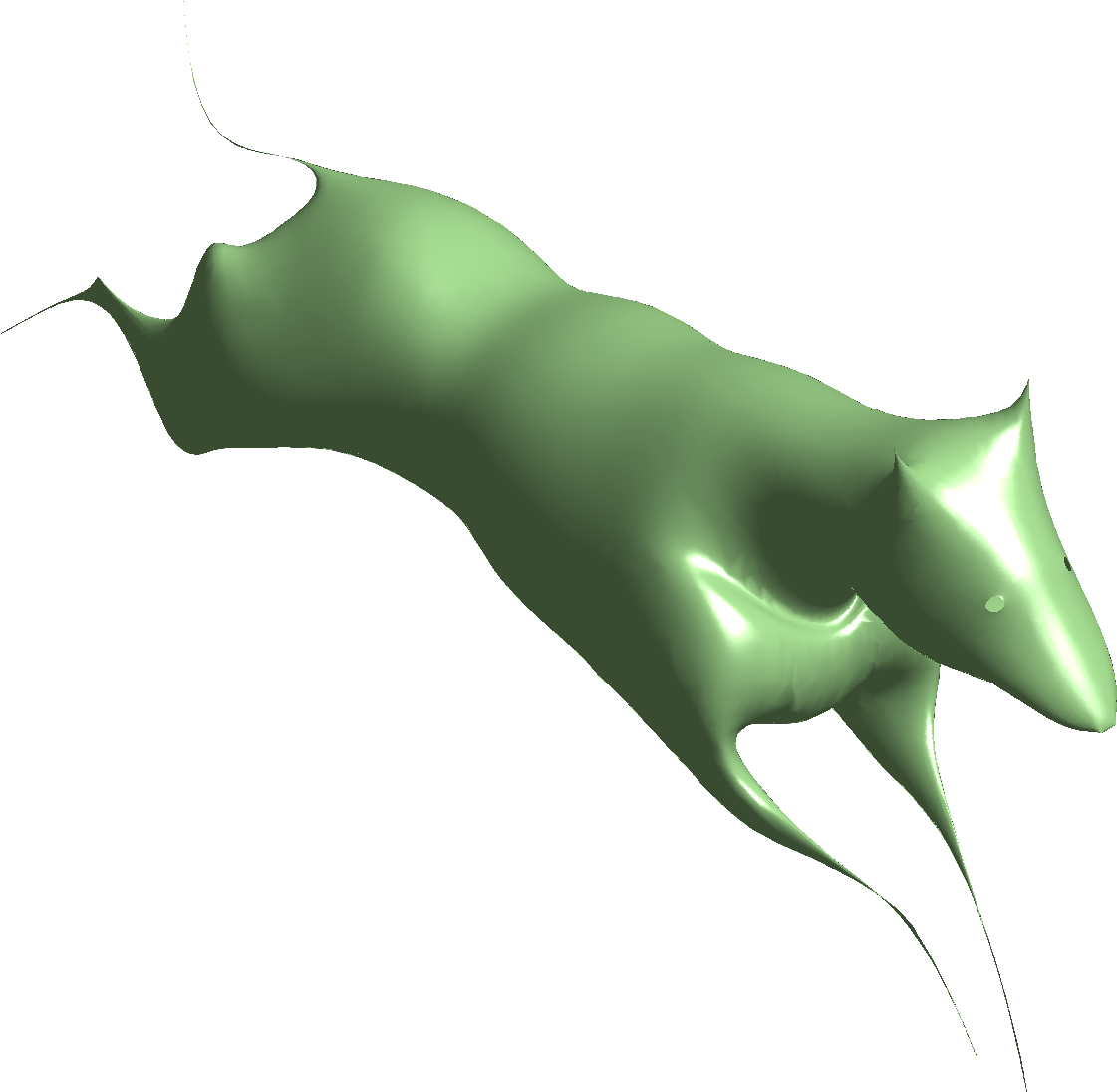}
\includegraphics[width=0.14\columnwidth]{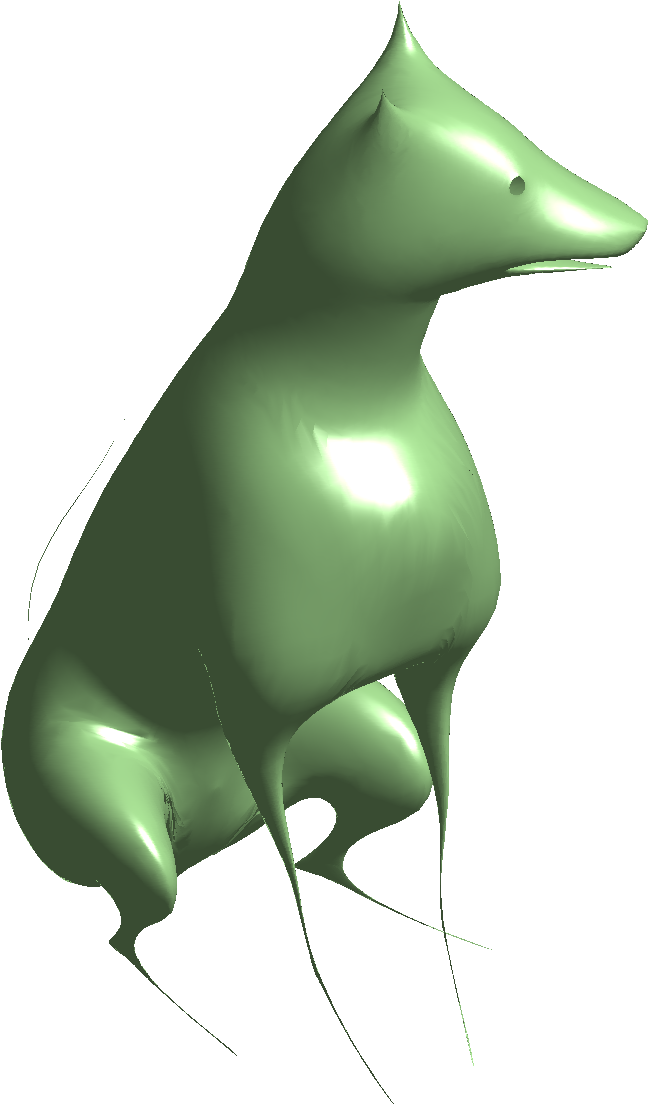}}
\fbox{
\includegraphics[width=0.20\columnwidth]{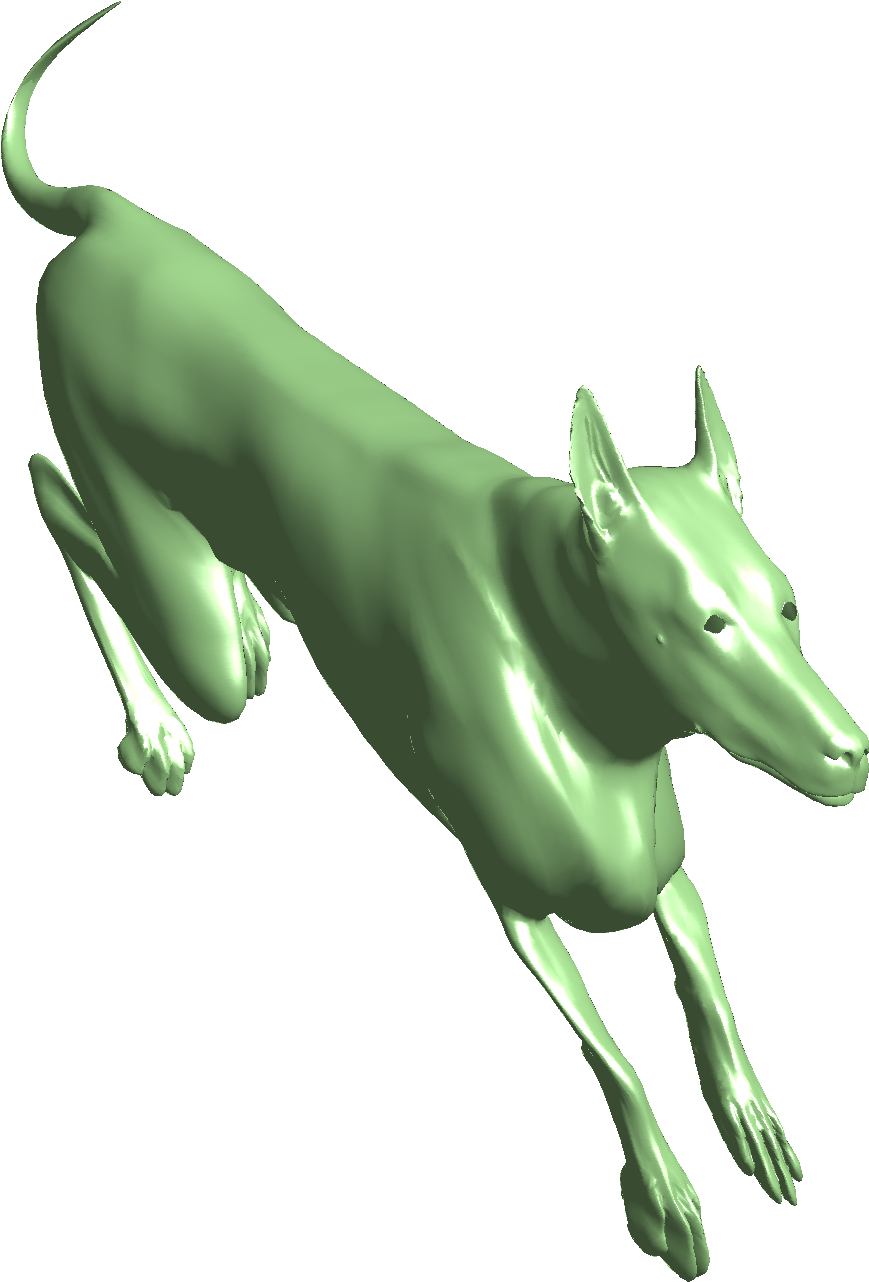}
\includegraphics[width=0.25\columnwidth]{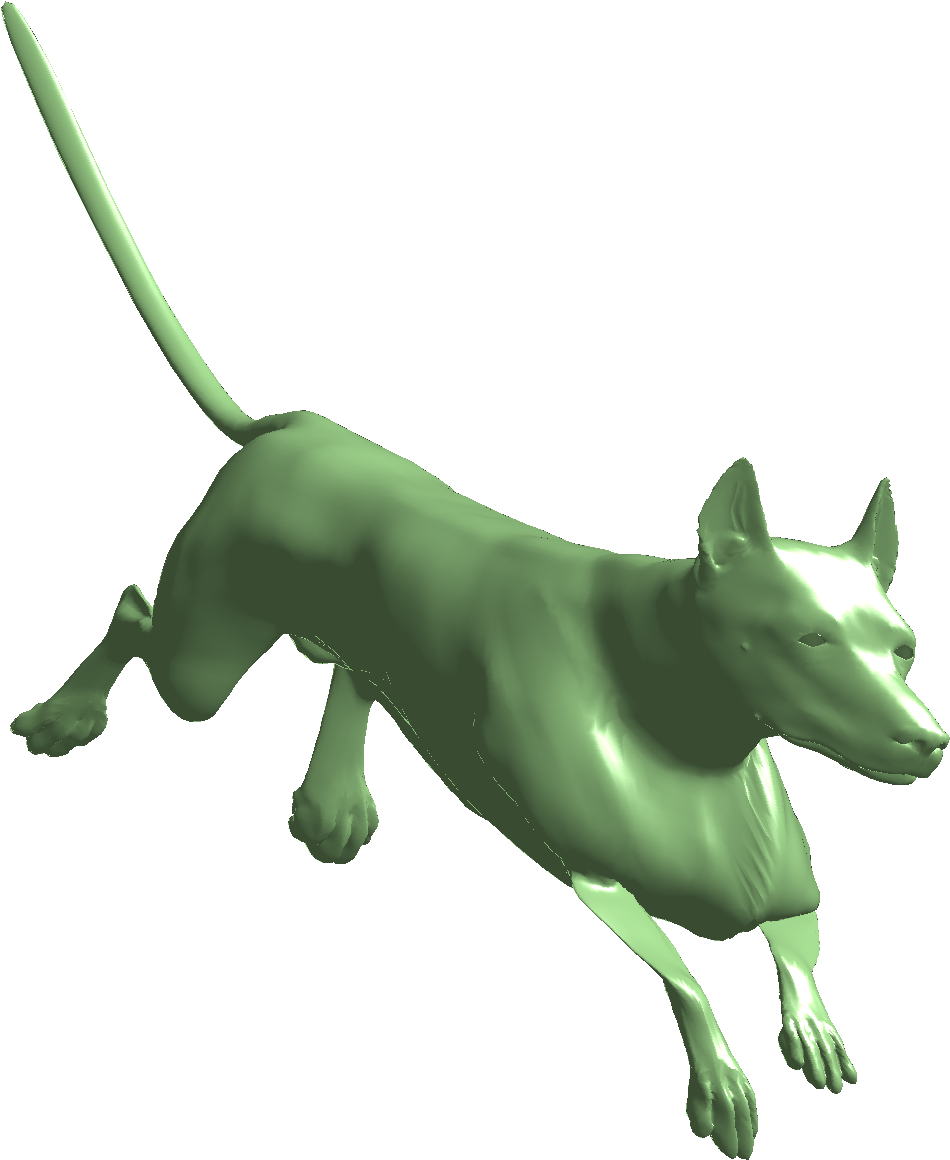}
\includegraphics[width=0.19\columnwidth]{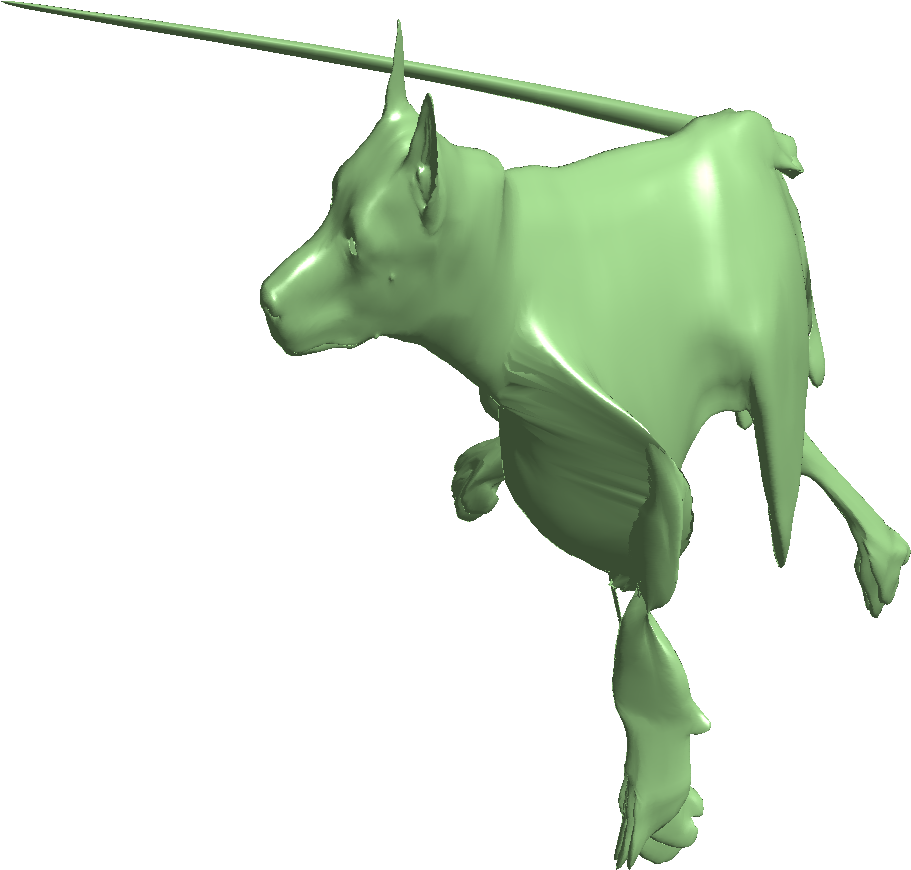}
\includegraphics[width=0.22\columnwidth]{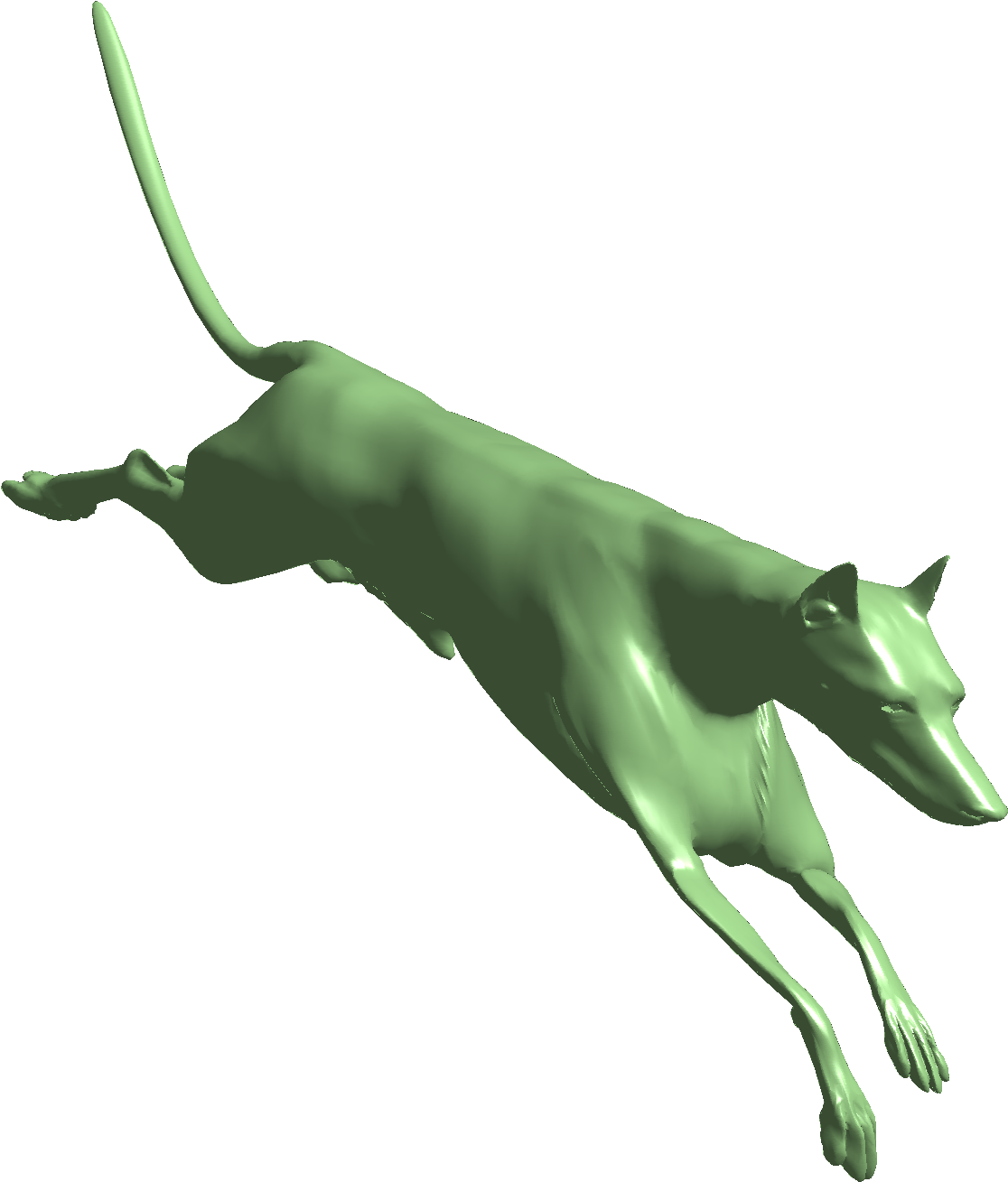}
\includegraphics[width=0.14\columnwidth]{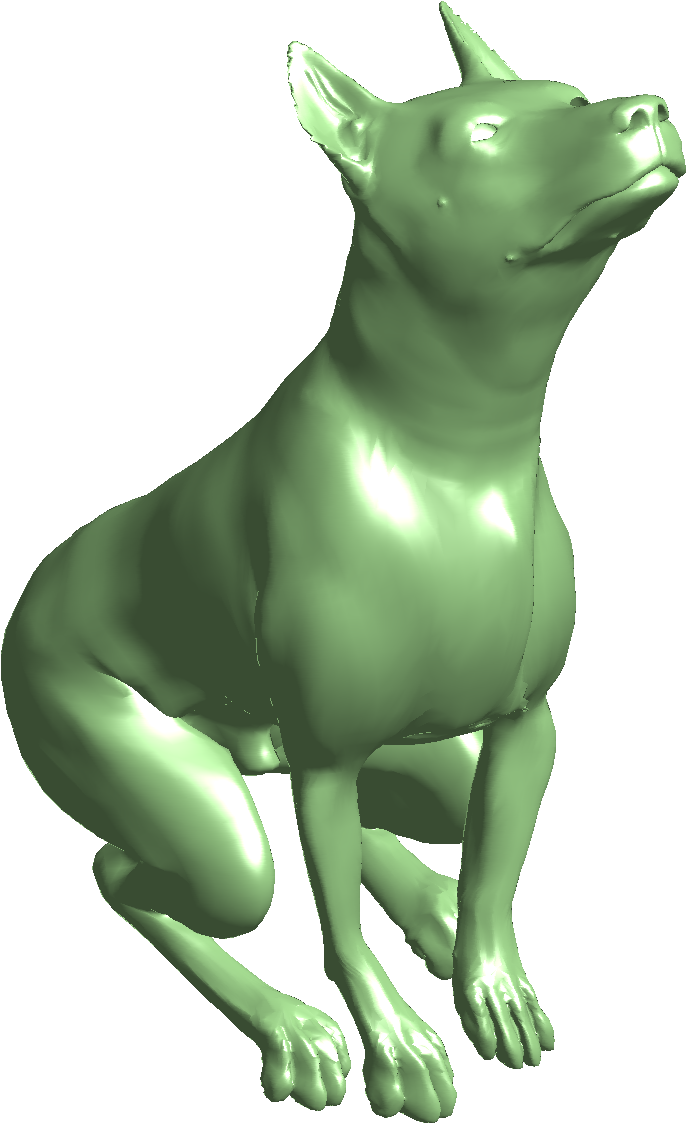}}
\fbox{
\includegraphics[width=0.22\columnwidth]{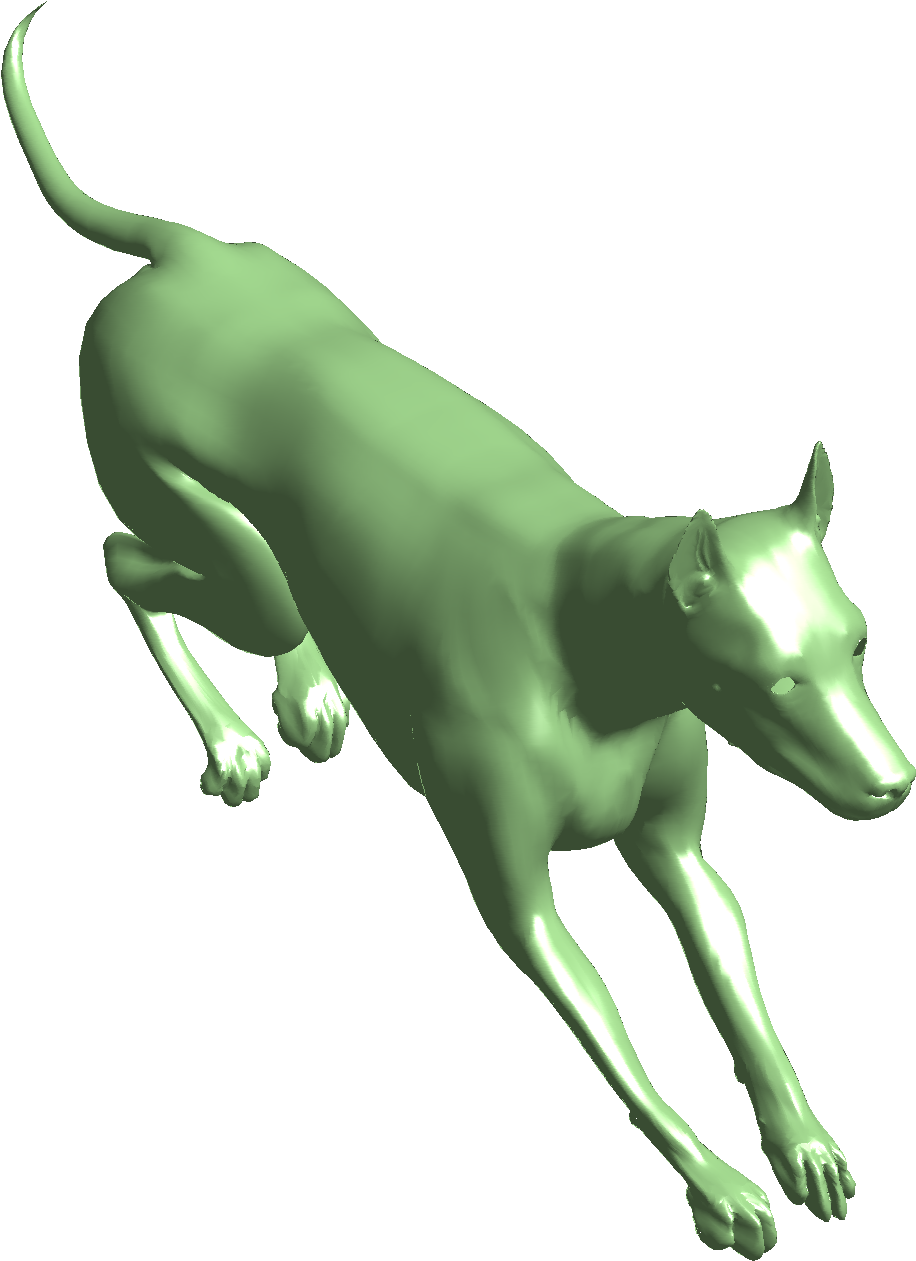}
\includegraphics[width=0.27\columnwidth]{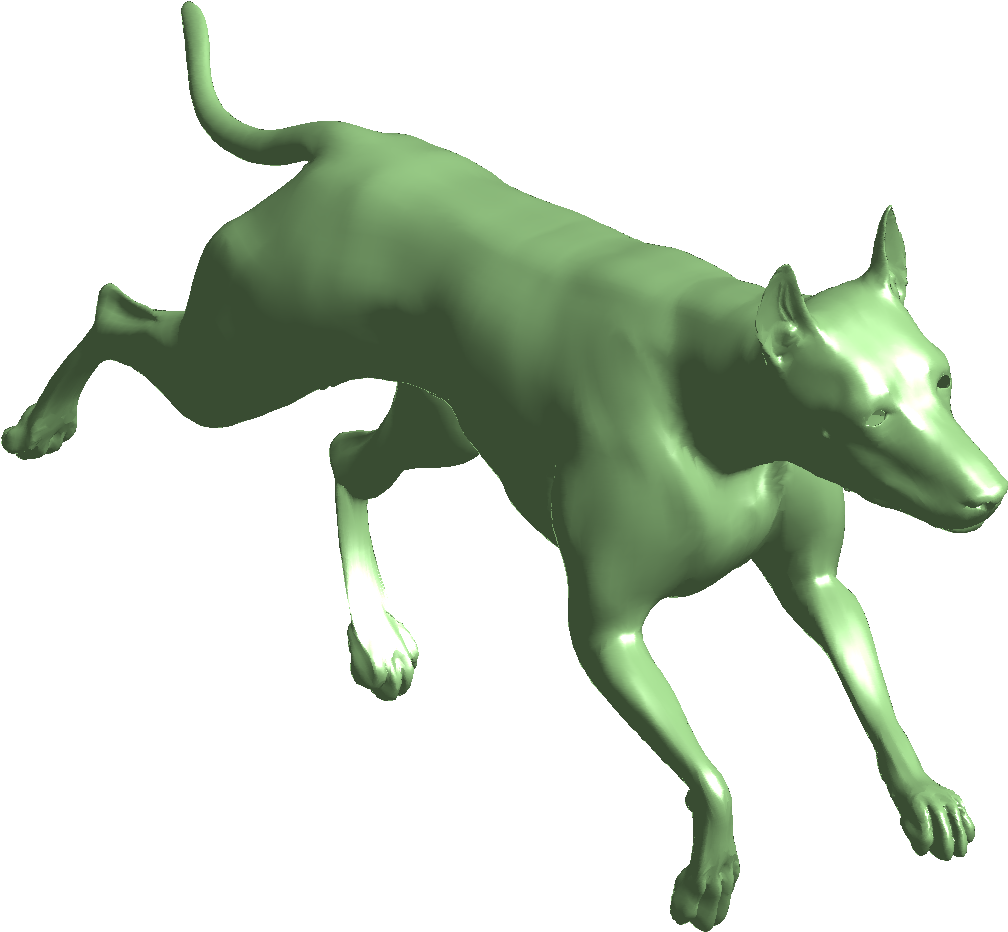}
\includegraphics[width=0.15\columnwidth]{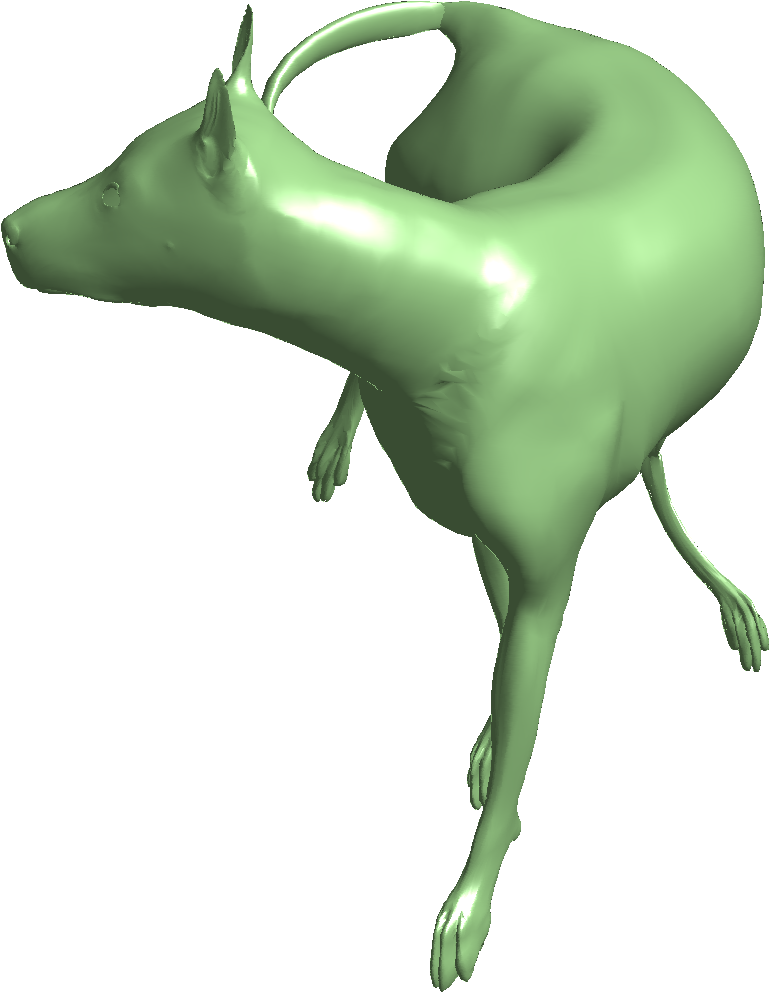}
\includegraphics[width=0.22\columnwidth]{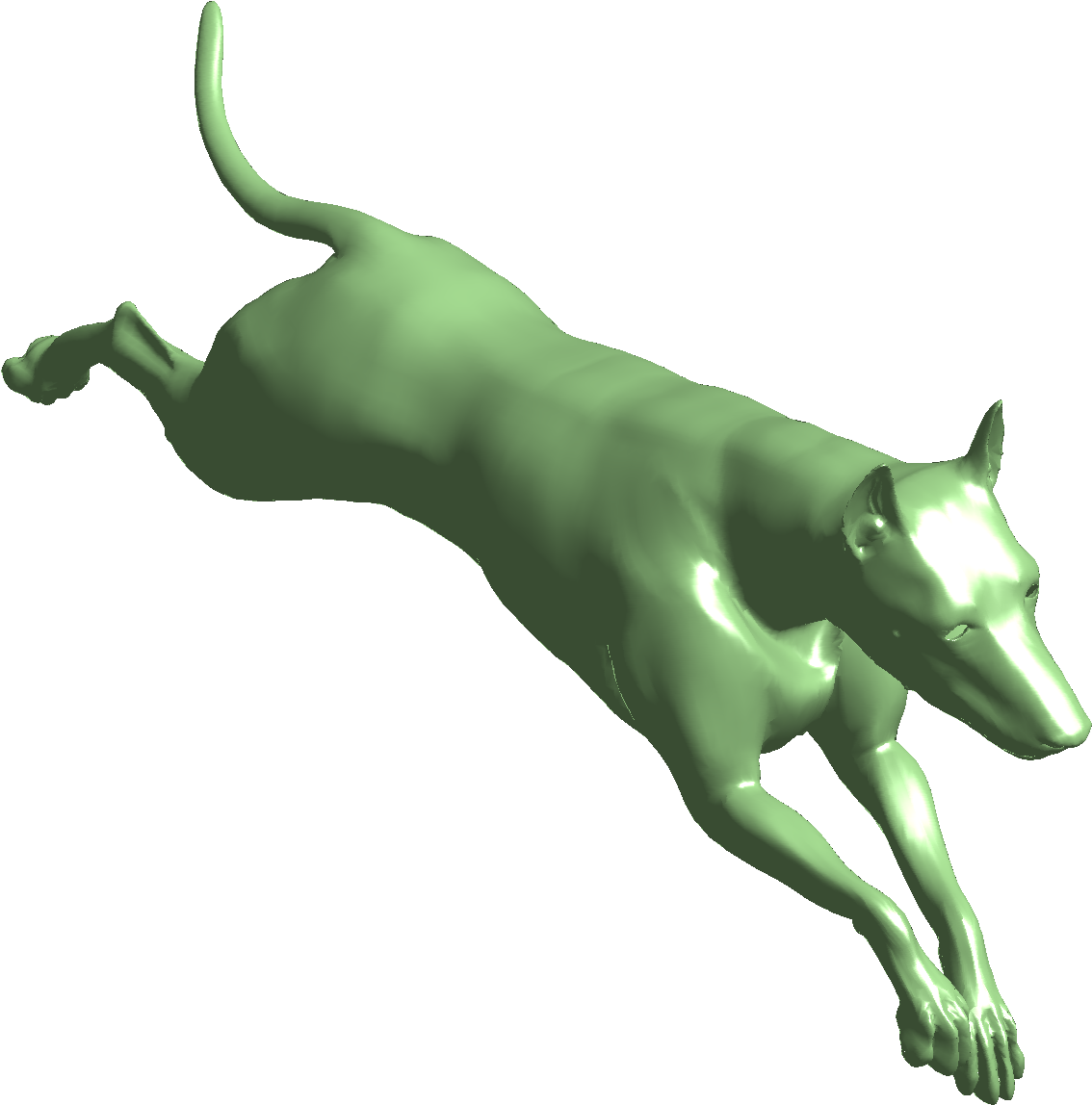}
\includegraphics[width=0.14\columnwidth]{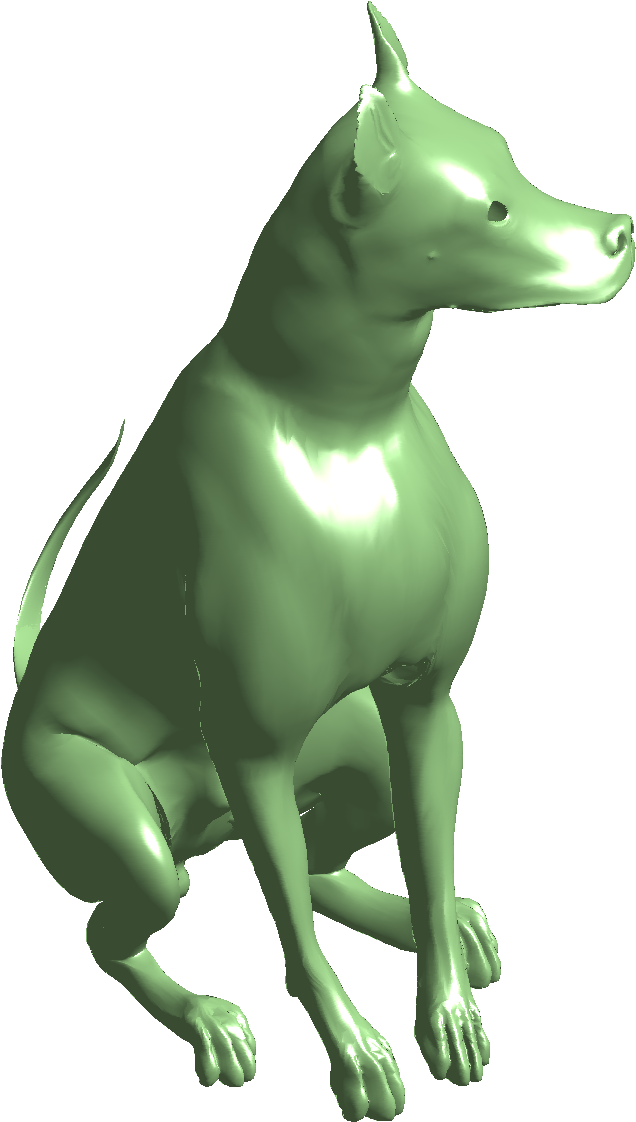}}
\caption{Reconstruction the dog shapes in the second row, by projecting their coordinates
  to the first  $100$ eigenfunctions  of the Laplace-Beltrami  (third row), 
  to the PCA basis trained with the two shapes from the first row (forth row),  
  and to the $100$ basis functions of the regularized-PCA basis trained with 
  the two dogs from the first row (bottom/fifth row).}
\label{fig:RPCA}
\end{figure*}

%
This model allows us to design an alternative basis which is 
 related  to the spectral domain but whose properties can be tuned  
 to fit specific information about the data.
 
\section{Conclusion}
A theoretical support for the selection 
 of the leading eigenfunctions of the Laplace-Beltrami operator of a given shape
 as a natural domain for morphometric study of surfaces was provided.
The optimality result motivates the design of efficient 
 shape matching and analysis algorithms.
It enabled us to find the most efficient representations of smooth functions on surfaces 
 in terms of both accuracy and complexity when projected onto the first eigenfunctions
  of the Laplace-Beltrami operator.
Our optimality criterion is obviously defined with respect to a given metric. 
In shape representation, the choice of an appropriate metric is probably as important 
 as the selection of the most efficient sub-space.
This was demonstrated in approximating the fine details and the general structure 
 of a shape of a horse in Section \ref{sec:scale}  using a regular metric, 
 a scale invariant one, and a metric interpolating between the two.
Spectral classical  scaling and its generalized version, benefit from the  
 presented optimality result, so as the regularization of classical PCA.
In both cases it was demonstrated that the decomposition of the LBO provides 
 a natural sub-space to operate in. 
 
The provably optimal representation space allows to construct efficient tools 
 for computational morphometry - the numerical study of shapes.
Revisiting the optimality criteria obviously lead to alternative domains 
 and hopefully better analysis tools that we plan to explore in the future.

\section{Acknowledgements}
YA and RK thank Anastasia Dubrovina, Alon Shtern, Aaron Wetzler, and Michael Zibulevsky for intriguing 
 discussions.
This research was supported by the European Community's FP7- ERC program, 
   grant agreement no. 267414.
HB was partially supported by NSF grant DMS-1207793 and also by grant number 238702 of the European Commission (ITN, project FIRST).

\appendix
\section*{Appendix}\label{App:appendix}
 \begin{theorem}
 \label{thrm:main}
Given a Riemannian manifold $S$ with a metric $(g_{ij})$, the induced LBO, $\Delta_g$,
 and its spectral basis $\phi_i$, where $\Delta_g\phi_i=\lambda_i\phi_i$, 
  and a real scalar value $0\leq\alpha<1$, there is no orthonormal basis of
  functions $\{\psi_i\}_{i=1}^\infty$, and an integer $n$ such that 
$$
 \left\|f-\sum_{i=1}^n\langle f, \psi_i \rangle \psi_i\right\|_2^2\leq
      \alpha\frac{\|\nabla_g f\|_2^2}{\lambda_{n+1}}, \,\,\,\,\,\,\,\,\,\, \forall f.
$$
\end{theorem}
%
%
To prove the optimality of the LBO eigenbasis, let us first prove the following lemma.
\begin{lemma}
Given an orthonormal basis $B=\{b_1,b_2,\ldots\}$, of a Hilbert space $\mathcal{V}$, an orthogonal projection operator $P$ of $\mathcal{V}$,  such that 
$$
\|P v\|\leq k\|v\|,~~~\forall v\in\spn{\{b_i,~~1\leq i\leq n\}}
$$
where $0<k<1$,
then 
$$
\dim(\ker(P))\geq n.
$$
\end{lemma}
\begin{proof}
Let us denote 
$$
\begin{disarray}{ll}
\mathcal{B}_1&=\{b_i, \|P b_i\|<1\}\\
\mathcal{B}_2&=\{b_i, \|P b_i\|=1\}\\
\mathcal{P}_1&=\ker(P)\\
\mathcal{P}_2&=\im(P).
\end{disarray}
$$
Because the operator $P$ is orthogonal and the basis $B$ is orthonormal, we have
$$
\mathcal{V}=\mathcal{B}_1\bigoplus \mathcal{B}_2=\mathcal{P}_1\bigoplus \mathcal{P}_2,
$$
and
$$
\begin{disarray}{ll}
\mathcal{B}_1^\perp&=\mathcal{B}_2\\
\mathcal{P}_1^\perp&=\mathcal{P}_2.
\end{disarray}
$$
By definition, we have that
$$
\mathcal{B}_2\subset\mathcal{P}_2.
$$
Then,
$$
\mathcal{B}_2^\perp\supset\mathcal{P}_2^\perp,
$$
and since $\mathcal{B}_2^\perp=\mathcal{B}_1$ and $\mathcal{P}_2^\perp=\mathcal{P}_1$, we have
$$
\mathcal{P}_1 \subset \mathcal{B}_1.
$$
Now, assume that
$$
\dim(\mathcal{P}_1)=\dim(\ker(P))<n \le \dim \mathcal{B}_1 
$$
Then, $\mathcal{P}_1\neq  \mathcal{B}_1$, and we can find a vector $u\in \mathcal{P}_1^\perp$ such that
 $\|u\|=1$ and  $u\in \mathcal{B}_1$. Since $\mathcal{P}_1^\perp=\ker(P)^\perp=\mathcal{P}_2$, 
  it follows that
$$
\|Pu\|=1.
$$ 
But, this contradicts the fact that $u\in \mathcal{B}_1$, because $u\in \mathcal{B}_1$ implies
$$
\|Pu\|<1.
$$
Then, 
$$
\dim(\ker(P))\geq n.
$$
\end{proof}
Equipped with this result we are now ready to prove Theorem A.1
\begin{proof}
Assume that there exists such a basis,  $\{\psi_i\}$. Then, the representation of a
 function $f$ in the eigenbasis of the LBO can be written as
$$
f=\sum_{i=1}^\infty \langle f,\phi_i\rangle\phi_i=\sum_{i=1}^\infty \beta_i\phi_i.
$$
We straightforwardly have   
$$
\|\nabla_g f\|_2^2=\sum_{i=1}^\infty \lambda_i \beta_i^2,
$$
and it follows that
$$
\alpha\frac{\|\nabla_g f\|_2^2}{\lambda_{n+1}}=\sum_{i=1}^\infty \underbrace{\frac{\alpha\lambda_i}{\lambda_{n+1}}}_{\tilde\lambda_i} \beta_i^2
 =\sum_{i=1}^\infty \tilde\lambda_i \beta_i^2.
$$
Moreover, 
$$
\left\|f-\sum_{i=1}^n\langle f,\psi_i \rangle \psi_i\right\|_2^2\leq\alpha\frac{\|\nabla_g f\|_2^2}{\lambda_{n+1}}\leq\sum_{i=1}^\infty \tilde\lambda_i \beta_i^2.
$$
Then, replacing $f$ with $\sum_{j=1}^{n+1}\beta_j\phi_j$, 
we have
$$
\left\|\sum_{j=1}^{n+1}\beta_j\phi_j-\sum_{i=1}^n\left\langle \sum_{j=1}^{n+1}\beta_j\phi_j,\psi_i \right\rangle \psi_i\right\|_2^2
  \le
  \sum_{j=1}^{n+1}\beta_j^2\tilde\lambda_j
  \le
  \left ( \max_{j=1}^{n+1}\tilde\lambda_j\right )\left (  \sum_{j=1}^{n+1}\beta_j^2 \right ), 
$$
 and since $\tilde\lambda_i<1,~~\forall i,~~ 1\leq i\leq n+1$, we can state that there is a set of $n+1$ orthonormal vectors $\phi_i$ belonging to an orthonormal basis whose projection error (over the space spanned by $\psi_i$) is smaller than one.
According to the previous lemma, the original assumption leads to a contradiction
  because the dimension of the kernel of the projection on the space spanned by $\psi_i$,
  $\, 1\le i\le n$,  is $n$. 
\end{proof}
%

\bibliographystyle{abbrv} 
\bibliography{biblio_full,mybib}{}
 \end{document}